\documentclass[11pt]{article}
\usepackage{fullpage}
\usepackage{amsmath,amsthm,amsfonts,amssymb,dsfont,bm}
\usepackage{enumerate,color,xcolor}
\usepackage{graphicx}
\usepackage{subfigure}
\usepackage[colorlinks,linkcolor=cyan,citecolor=blue]{hyperref}
\usepackage{url}

\usepackage{natbib}

\usepackage{tikz}
\usepackage{subcaption}

\numberwithin{equation}{section}
\theoremstyle{plain}

\newtheorem{theorem}{Theorem}
\newtheorem{corollary}[theorem]{Corollary}

\newtheorem{lemma}[theorem]{Lemma}

\newtheorem{assumption}[theorem]{Assumption}

\newcommand{\mk}{\mathcal{K}}
\newcommand{\mc}[1]{\mathcal{#1}}
\newcommand{\R}{\mathbb{R}}
\newcommand{\paren}[1]{\left(#1\right)}
\newcommand{\proj}{\mathcal{P}_{\mathcal{K}}}

\newcommand{\prox}{q_{\mk}^{\lambda}}
\newcommand{\sqb}[1]{\left[#1\right]}
\newcommand{\bsqb}[1]{\Bigg[#1\Bigg]}
\newcommand{\mb}[1]{\mathbf{#1}}
\newcommand{\norm}[1]{\left\|#1\right\|}
\newcommand{\normsq}[1]{\left\|#1\right\|^2}
\newcommand{\Wc}{\ensuremath{\mathcal{W}}}
\newcommand{\E}{\mathbb{E}}
\newcommand{\Law}{\mathcal L}
\newcommand{\lrangle}[1]{\left\langle#1\right\rangle}
\newcommand{\ug}{U^{\gamma}}
\newcommand{\ugc}{U^{\gamma}_{\Wc,\eta}}

\begin{document}
\begin{center}  
{\bf\Large Decentralized Proximal Stochastic Gradient Langevin Dynamics}
\end{center}

\author{}
\begin{center}
    {Mohammad Rafiqul Islam}\,\footnote{Department of Mathematics, Florida State University, Tallahassee, Florida, United States of America;
  rislam@fsu.edu},
  Lingjiong Zhu\,\footnote{Department of Mathematics, Florida State University, Tallahassee, Florida, United States of America; zhu@math.fsu.edu
 }
\end{center}

\begin{center}
 \today
\end{center}

\begin{abstract}
We propose Decentralized Proximal Stochastic Gradient Langevin Dynamics (DE-PSGLD), a decentralized Markov chain Monte Carlo (MCMC) algorithm for sampling from a log-concave probability distribution constrained to a convex domain. Constraints are enforced through a shared proximal regularization based on the Moreau–Yosida envelope, enabling unconstrained updates while preserving consistency with the target constrained posterior. We establish non-asymptotic convergence guarantees in the 2-Wasserstein distance for both individual agent iterates and their network averages. Our analysis shows that DE-PSGLD converges to a regularized Gibbs distribution and quantifies the bias introduced by the proximal approximation. We evaluate DE-PSGLD for different sampling problems on synthetic and real datasets. As the first decentralized approach for constrained domains, our algorithm exhibits fast posterior concentration and high predictive accuracy. The code for the numerical experiments is publicly available at \href{https://anonymous.4open.science/r/Decentralized-proximal-SGLD-74E8}{anonymous.4open.science/r/Decentralized-proximal-SGLD-74E8}.
\end{abstract}

\section{Introduction}
Decentralized learning is a learning process in which data is distributed across computational agents or collected by individual agents, and model parameters are computed as the consensus of the agents. 
It has gained a lot of interest for applications where agents can collaboratively learn a predictive model without sharing their own data, but sharing only their local models with their immediate neighbors to generate a global model \citep{he2018cola,hendrikx2019accelerated,arjevani2020ideal}.
We assume there are $N$ agents who are connected over an undirected communication network $\mathcal{G}=(\mathcal{V,E})$ where $\mathcal{V}=\{1,\ldots,N\}$ represents the agents and $\mathcal{E}\subseteq \mathcal{V\times V}$ denotes the set of edges; i.e., if agent $i$ and $j$ are connected then $(i,j) \in \mathcal{E}$ implies $(j,i) \in \mathcal{E}$. Suppose we have a collection of $n$ independent and identically distributed (i.i.d.) data pairs $z_i=(a_i, y_i)$, where $a_i\in\R^p$ is the feature vector and $y_i$ the label or response of the $i$-th observation. Let $Z=[z_1, z_2,\cdots,z_n]\in \mathbb{R}^{np }$ be sampled from the distribution $p(Z|x)$ where the parameter $x\in \mathbb{R}^d$ has a common prior. The goal is to sample from the posterior distribution $p(x | Z) \propto p(Z|x) p(x)$ by distributing $Z$ among $N$ agents such that $Z_i=\{z_1^i, z_2^i,\cdots, z_{n_i}^i\}$ is the subset of data exclusive to agent $i$. Thus, $Z=\cup_{i=1}^{N}Z_i$ and $Z_i\cap Z_j =\emptyset$ for $i\ne j$. Therefore, by the additive property of the log-likelihood function, we have $\log p(Z|x) = \sum_{i=1}^{N}\sum_{j=1}^{n_i}\log p(z_j^i|x)$. By defining the potential function $f(x)$ as
\begin{equation}
    \label{eq:mainprob2}
    f(x):= \sum_{i=1}^{N}f_i(x), \qquad f_i(x):= -\sum_{j=1}^{n_i} \log p(z_j^i |x) -\frac1N\log p(x),
\end{equation}
we set our goal to sample from the convex constrained set $\mk\subseteq \R^d$, with the probability density function $\pi(x): = p(x|Z)\propto e^{-f(x)}$, where each component function $f_i$ is processed with the data that are exclusive to agent $i$ only. The choice of the log-likelihood function depends on the type of the learning problem, \emph{e.g.,} Bayesian linear regression \citep{hoff2009first}, Bayesian logistic regression \citep{hoff2009first}, Bayesian deep learning \citep{wang2016towards,polson2017}
and Bayesian principal component analysis \citep{dubey2016variance}.

\section{Related Work}
Decentralized optimization has been extensively studied in recent years to address large-scale learning problems where centralized computation is either infeasible or undesirable due to communication, privacy, or robustness constraints. A wide variety of decentralized first-order algorithms have been proposed, including decentralized gradient descent, EXTRA-type methods, decentralized stochastic gradient descent (DSGD), decentralized accelerated proximal stochastic gradient (DAPSGD), and accelerated variants; see, for example,~\citep{ye2020decentralized, arjevani2020ideal, shi2015extra, lian2017can}. Motivated by stochastic optimization literature, there has been a growing interest in decentralized \textit{sampling} algorithms. Decentralized Langevin-type algorithms extend classical Langevin Monte Carlo by allowing each agent to access only its local data while communicating with neighbors to approximate the global posterior. Representative works include decentralized stochastic gradient Langevin dynamics (DE-SGLD), EXTRA-SGLD, DIGing-SGLD, and their accelerated variants \citep{gurbuzbalaban2021decentralized,gurbuzbalaban2024generalized,DIGing,yao2025accelerating}. These methods are primarily designed for unconstrained sampling on $\mathbb{R}^d$ and rely on consensus mechanisms to ensure asymptotic agreement across agents. Sampling from probability distributions supported on constrained domains or involving non-smooth potentials has attracted significant attention in the centralized setting. One classical approach is projected Langevin Monte Carlo (PLMC), studied in~\citet{bubeck2015finite,bubeck2018sampling}, where each Langevin step is followed by a projection onto the feasible set:
\begin{equation}
    x_{k+1}=\mathcal{P}_{\mathcal{K}}\left(x_{k}-\eta\nabla f(x_{k})+\sqrt{2\eta}\xi_{k+1}\right),
    \label{projected:Langevin}
\end{equation}
with $\mathcal{P}_{\mathcal{K}}$ denoting the Euclidean projection onto a convex set $\mathcal{K}$, and the dynamics \eqref{projected:Langevin} is based on the discretization of the continuous-time overdamped Langevin stochastic differential equation (SDE) with reflected boundary: 
\begin{equation}
    dX_{t}=-\nabla f(X_{t})dt+\sqrt{2}dW_{t}+\nu(X_{t})L(dt),
\end{equation}
where the term $\nu(X_{t})L(dt)$ ensures that $X_{t}\in\mathcal{K}$ 
for every $t$ given that $X_{0}\in\mathcal{K}$. In particular, $\int_{0}^{t}\nu(X_{s})L(ds)$ is a bounded variation reflection process
and the measure $L(dt)$ is such that $L([0,t])$ is finite, 
$L(dt)$ is supported on $\{t|X_{t}\in\partial\mathcal{K}\}$. Non-asymptotic convergence rates in total variation distance were established in~\citet{bubeck2018sampling}, though with a strong dependence on the underlying dimension. Projected stochastic gradient Langevin dynamics has also been studied under weaker assumptions on the objective function. In \cite{Lamperski2021}, the author analyzed projected SGLD for possibly non-convex smooth objectives with stochastic gradients satisfying sub-Gaussian noise conditions, obtaining convergence guarantees in the $1$-Wasserstein distance; see also \cite{zheng2022constrained} for related developments. Mirror descent-based Langevin algorithms (see e.g. \cite{hsieh2018mirrored,Chewi2020,Zhang2020,TaoMirror2021,Ahn2021}) form another important class for constrained sampling. Mirrored Langevin dynamics was proposed in \cite{hsieh2018mirrored}, inspired by the classical mirror descent in optimization. An alternative strategy to handling constraints is through proximal and penalty-based methods. Proximal Langevin Monte Carlo was introduced in \cite{brosse2017sampling}, where the non-smooth component of the potential is handled via Moreau--Yosida regularization. This approach avoids explicit projections and yields polynomial complexity bounds under log-concavity assumptions. Subsequent work by \citet{SR2020} further examined proximal SGLD from a primal--dual perspective. More recently, penalized Langevin Monte Carlo methods have been proposed as an alternative to projection- or proximal-based approaches. Inspired by classical penalty methods in optimization, \cite{gurbuzbalaban2024penalized} introduced penalized Langevin dynamics for constrained sampling with potentially non-convex objectives and demonstrated improved dimension dependence compared to earlier methods. Related work combining constraints with replica-exchange Langevin dynamics was studied in \cite{constraint_replica}, showing that reducing the effective domain diameter can significantly accelerate mixing. In addition, motivated by the acceleration properties of non-reversible Langevin dynamics on $\mathbb{R}^d$ \cite{HHS93,HHS05,FSS20,GGZ2,HWGGZ20}, recent studies have begun to explore non-reversible dynamics for constrained sampling. In particular, \cite{DFTWZ2025} proposed skew-reflected non-reversible Langevin dynamics and established non-asymptotic convergence guarantees, showing that breaking reversibility can yield acceleration even in the presence of constraints. Further theoretical insights based on large deviations and asymptotic variance, along with extensive numerical experiments, are provided in \cite{WTWZ2025}.

Despite these advances, the literature on decentralized sampling and constrained Langevin methods has developed largely independently. To the best of our knowledge, existing decentralized Langevin algorithms are limited to smooth, unconstrained settings, while constrained and proximal Langevin methods rely on centralized computation. In particular, there is no prior work that systematically studies \emph{decentralized proximal or penalized stochastic gradient Langevin dynamics} for sampling from constrained or composite posterior distributions. Addressing this gap is the main motivation of the present work. The contributions of our paper can be summarized as follows:
\begin{itemize}
    \item We propose a new sampling algorithm that can be used to generate samples from a constrained domain in a decentralized system. To the best of our knowledge, this is the first decentralized Langevin sampling algorithm in a constrained domain.
    \item We provide a non-asymptotic convergence analysis in the 2-Wasserstein distance for both individual agent and network consensus chain (Theorem~\ref{thm:mt}). Based on this, we provide an iteration complexity for a given accuracy level (Corollary~\ref{cor:revised}).
    \item Finally, we show the efficiency of our algorithm for sampling with synthetic and real data for different sampling problems, including Bayesian linear and logistic regressions. 
\end{itemize}

\section{Background and Problem Setup}
\subsection{Langevin Dynamics} 
\textit{Langevin algorithms} are one of the most widely used Markov Chain Monte Carlo (MCMC) methods in statistical learning that allow sampling from a given density $\pi(x)$ of interest. The classical \textit{Langevin algorithm} is based on the discretization of \textit{overdamped Langevin SDE} \citep{dalalyan2017theoretical,durmus2017nonasymptotic}:
\begin{equation}
    dX(t)=-\nabla f(X(t))dt+\sqrt{2}dW_t,
    \label{eq:ovdLangevin}
\end{equation}
where $f:\mathbb{R}^d\rightarrow \mathbb{R}$ and $W_t$ is a standard $d-$dimensional Brownian motion with $W_0=0$. Under some mild assumptions on $f$, the diffusion \eqref{eq:ovdLangevin} admits a unique stationary distribution with the density $\pi(x)\propto e^{-f(x)}$, also known as the \textit{Gibbs distribution} \citep{pavliotis2014stochastic}. The implementation of this algorithm requires the discretized version of the dynamics, and Euler-Maruyama discretization is the simplest one, known as \textit{Unadjusted Langevin Algorithm} (ULA) \citep{durmus2017nonasymptotic,durmus2019high}:
\begin{equation}
    x_{k+1} = x_k -\eta \nabla f(x_k)+ \sqrt{2\eta}w_{k+1},
    \label{eq:em_ovdLangevin}
\end{equation}
where $\eta>0$ is {\color{blue}the step-size (or learning rate)}, and $w_{k}\in \mathbb{R}^d$ is a sequence of independent and identically distributed (\emph{i.i.d.}) standard Gaussian random vectors $\mathcal{N}(0,\mathbf{I}_d)$. However, the discretization \eqref{eq:em_ovdLangevin} does not converge to the target distribution $\pi(x)$, and it introduces a bias that needs to be properly characterized to ensure performance guarantees \citep{dalalyan2019user}. The \textit{Unadjusted Langevin Algorithm} (ULA) \eqref{eq:ovdLangevin} also requires the computation of $\nabla f$ at each iteration, which can be computationally expensive and often impractical when the data are large and multi-dimensional. This issue can be handled efficiently using \textit{stochastic gradient}, instead of the full gradient \citep{bottou2010large} which results in the algorithm \textit{Stochastic Gradient Langevin Dynamics } (SGLD) \citep{welling2011bayesian,Raginsky} given as 
\begin{equation}
    x_{k+1} = x_k -\eta \tilde{\nabla}f(x_k)+\sqrt{2\eta}w_{k+1},
    \label{eq:sgld}
\end{equation}
where $\eta>0$ is {\color{blue}the step-size (or learning rate)}, $w_{k}\in \mathbb{R}^d$ is a sequence of \emph{i.i.d.} standard Gaussian random vectors $\mathcal{N}(0,\mathbf{I}_d)$, and the stochastic gradient $\tilde{\nabla}f(x_k)$ is an unbiased estimator of the deterministic gradient with a bounded variance. When the number of data points is large, these stochastic gradients are cheaper to estimate and can be computed from a mini-batch setup.
\subsection{Decentralization}
We have a collection of $N$ computational agents connected over a communication network with the communication matrix $W$, where $W=[W_{i,j}]\in \mathbb{R}^{N\times N}$ is symmetric, doubly stochastic matrix with the properties, $W_{ij}=W_{ji}>0$ for $i\ne j$ if $\{i,j\}\in \mathcal{E}$, $W_{ij}=W_{ji}=0$ for $\{i,j\} \notin \mathcal{E}$, and $W_{ii}=1 - \sum_{j\ne i}W_{ij}>0$ for every $1\le i \le N$. Moreover, the eigenvalues of $W$ satisfy $1=\lambda_1^W > \lambda_2^W >\cdots >\lambda_N^W>-1$, with $W\mathbf{1} = \mathbf{1}$ where $\mathbf{1}$ is a vector of length $N$ with entries equal to 1. A common approach to compute $W$ is taking $W=\mathbf{I}_N-\delta L$, where $\mathbf{I}_N$ is an $N\times N$ identity matrix, $L$ is the graph Laplacian, and $\delta>0$ is a small number satisfying $0<\delta < \frac{2}{\lambda_N^L}$ \citep{olfati2007consensus, chung1997spectral}. We define the spectral gap $1-\rho\in (0,1)$ where, $\rho:=\max\{|\lambda_2^W|,|\lambda_N^W|\}$ is the second largest eigenvalue of $W$.

Let $x_i^{(k)}$ be the local variable of the agent $i$ at $k$-th iteration. The Decentralized Stochastic Gradient Langevin Dynamics (DE-SGLD)~\citep{gurbuzbalaban2021decentralized} consists of a weighted averaging with local variables $x_j^{(k)}$ of node $i$'s immediate neighbors $j\in \Omega_i : =\{j: (i,j)\in \mathcal{G}\}$ as well as a stochastic gradient step over the node's component function $f_i(x)$, i.e.
\begin{align}
    x^{k+1}_i &=\sum\nolimits_{j\in \Omega_i} W_{ij}x_j^{(k)} -\eta \tilde{\nabla}f_i\left(x_i^{(k)}\right)+\sqrt{2\eta}w^{(k+1)}_i, \label{eq:desgld}
\end{align}
where $\eta>0$ is the stepsize, $\tilde{\nabla}f_i\paren{x_i^{(k)}}$ is the stochastic gradient, and $w_{i}^{(k)}\in \mathbb{R}^d$ are \emph{i.i.d.} standard Gaussian random vectors $\mathcal{N}(0,\mathbf{I}_d)$. The DE-SGLD algorithm is an unconstrained sampling algorithm for any $x\in\R^d$ and can be used for each agent $i$ to sample a common distribution $e^{-f(x)}$. We want to introduce regularization and constraints via a proximal function so that for any $x\in \R^d$, the iterations are onto a convex compact set $\mk$. We propose and discuss this technique in the next section.

\section{Main Results}
Let $\mk\in \R^d$ be a compact convex set with non-empty interior. To sample from a distribution $\pi$ restricted to a compact convex set $\mk$, we cannot use the ULA, and its discretized versions in either centralized \eqref{eq:sgld} or decentralized \eqref{eq:desgld} setting. For constraint sampling, we need to add some regularization in the form of a proximal function.

The unconstrained potential function $\displaystyle u:\R^d\rightarrow (-\infty,+\infty]$ of the form $u(x):=f(x) + q_{\mathcal{K}}(x)$, is associated with the probability density function (PDF) $\pi(x)\propto e^{-u(x)}$, and $q_{\mathcal{K}}(x)$ is the indicator function given as
\begin{equation}
    q_{\mathcal{K}}\,(x)=\begin{cases}
        +\infty&\quad\mathrm{if}\quad x\ne \mk,\\
        0&\quad\mathrm{if}\quad x\in \mk.
    \end{cases}
    \label{eq:indicator}
\end{equation}
For the constraint sampling, we use the \textit{Moreau-Yosida} envelope~\citep{rockafellar1998variational}, $q_{\mathcal{K}}^{\gamma}:\R^d \rightarrow \R_+$ of the form
\begin{equation}
    q_{\mathcal{K}}^{\gamma}(x)=\inf_{y\in \R^d} \paren{q_{\mathcal{K}}(y)+\frac{1}{2\gamma}\|x-y\|^2}=\frac{1}{2\gamma}\|x-\mathcal{P}_{\mk}(x)\|^2,
    \label{eq:mye}
\end{equation}
where $\displaystyle \proj(x):=\arg\min_{y\in \mk}\norm{x-y}$ denotes the projection onto $\mk$, and $\gamma>0$ is a regularization parameter. Then the constrained potential $u^{\gamma}:\R^d\rightarrow \R$ given by $u^{\gamma}(x):=f(x)+q_{\mathcal{K}}^{\gamma}(x)$, is associated with the stationary distribution $\pi^{\gamma}(x)\propto e^{-u^{\gamma}(x)}=\exp{\paren{-f(x)-\frac{1}{2\gamma}\|x-\proj(x)\|^2}}$, for any $x\in \R^d$.
\subsection{Distance Between $\pi^{\gamma}$ and $\pi$}
In this section, we bound the distance between $\pi^{\gamma}$ and $\pi$ in terms of the 2-Wasserstein distance. Let us define $\mb{P}_{2}(\R^d)$ as the space of all Borel probability measures $\mu$ on $\R^d$ with finite second moment (with respect to the Euclidean norm). For any two Borel probability measures $\mu,\nu$ in $\mb{P}_{2}(\R^d)$, the 2-Wasserstein distance is given as \citep{villani2008optimal}:
\begin{equation*}
\mathcal{W}_2^2(\mu,\nu):=\inf_{\gamma\in\Pi(\mu,\nu)}\int_{\R^d\times\R^d}\|x-y\|^2\,d\gamma(x,y),
\end{equation*}
where $\Pi(\mu,\nu)$ is the set of couplings of $(\mu,\nu)$, and $\mu$ is absolutely continuous with respect to $\nu$. 

We have the following lemma that 
upper bounds the 2-Wasserstein distance between $\pi^{\gamma}$ and $\pi$.

\begin{lemma}
    \label{lem:newlemma}
    Fix any $\gamma_0\in (0,1/e)$. Then, for any $0<\gamma\le \gamma_0$,
    \begin{equation}
        \Wc_{2}\paren{\pi^\gamma,\pi} \le \mc{C}\gamma^{1/8}(\log(1/\gamma))^{1/8},\label{eneq:dist}
    \end{equation}
    where $\mc{C}$ is a constant that depends on $\gamma_0$, and defined in~\eqref{eq:lebesgue5} in the proof given in Appendix~\ref{lem:proof-newlemma}. 
\end{lemma}

\subsection{Decentralized Proximal SGLD}
We are now ready to present our main algorithm that requires the following assumptions. Throughout this paper, we assume that the following assumptions hold.

\begin{assumption}
    \label{assumption:convex}
    We assume each component function $f_i:\mathbb{R}^d\rightarrow \mathbb{R}, \in \mathcal{S}_{\mu,L}$ is $\mu-$strongly convex and $L$-smooth with $L>\mu$. That is for any $f_i \in \mathcal{S}_{\mu,L}$, and for every $x_1, x_2 \in \mathbb{R}^d$, 
    \[
    \frac{\mu}{2}\|x_1-x_2\|^2 \le f_i(x_1)-f_i(x_2)-\nabla f_i(x_2)^\top(x_1-x_2) \le \frac{L}{2}\|x_1-x_2\|^2.
    \]  
\end{assumption}
Next, denote $B(x,r)$ as the closed ball centered at $x\in \R^d$ with radius $r>0$ such that $B(x,r):=\{y\in \R^d:\|y-x\|\le r\}$. 
We impose the following assumption on $\mathcal{K}$. 

\begin{assumption}
    \label{assumption:ball}
    There exist $r, R>0, r\le R$, such that, $B(0,r)\subset \mk\subset B(0,R)$.
\end{assumption}

Next, we impose the following assumption on the gradient noise.

\begin{assumption}
    \label{assump:gradnoise}
    The gradient noises are defined as $\xi_i^{(k+1)}:=\tilde{\nabla}f_i\left(x_i^{(k)}\right) - \nabla f_i\left(x_i^{(k)}\right)$ are unbiased with a finite second moment. That is, $\mathbb{E}\left[\xi_i^{(k+1)} \vrule\ \mathcal{F}_k\right]= 0,\,\text{and } \mathbb{E}\left\|\xi_i^{(k+1)}\right\|^2 \le \sigma^2$, where $\mathcal{F}_k$ is the natural filtration of the iterates $x_i^{(k)}$ up to and including time $k$.
\end{assumption}
We also use two important properties from~\cite{rockafellar1998variational} and~\cite{brosse2017sampling}. The proximal function $\prox$ is convex, continuously differentiable, and $\gamma^{-1}$-Lipschitz, \emph{i.e.,}
\[
\nabla \prox(x)= \gamma^{-1}\paren{x-\proj(x)},\qquad \norm{\nabla \prox(x)-\nabla \prox(y)}\le \gamma^{-1}\norm{x-y},\quad\forall\, x,y\in \R^d.
\]
Based on the setup and assumptions, we present the Decentralized Proximal Stochastic Gradient Langevin Dynamics (DE-PSGLD) as
\begin{equation}
    \begin{split}
        x_i^{(k+1)}&=\sum_{j\in \Omega_i} W_{ij}x_j^{(k)}-\eta\sqb{\nabla f_i\paren{x_i^{(k)}}+\frac{1}{N\gamma}\paren{x_i^{(k)}-\proj\paren{x_i^{(k)}}}}-\eta \xi_i^{(k+1)}+\sqrt{2\eta}w_i^{(k+1)},
    \end{split}
    \label{eq:dpsgld}
\end{equation}
where $\eta>0$ is the stepsize, $\xi_{i}^{(k)}$ are stochastic gradient noise satisfying Assumption~\ref{assump:gradnoise}, $w_i^{(k)}$ are i.i.d. Gaussian random vectors $\mathcal{N}(0,\mathbf{I}_d)$ independent of the stochastic gradient noise, and $\Omega_i = \{j: (i,j)\in \mathcal{G}\}$ are the neighbors of the node $i$.

\subsection{Convergence Analysis}
Next, we concatenate the local decision variables into a single vector to facilitate convergence analysis. For $\mk\subset \R^d$, we define the product set $\mk^N:=\mk\times\cdots \times\mk \subset \R^{Nd}$, and let $x^{(k)}$ define the decision vector from all agents
\begin{equation}
    x^{(k)} = \left[\left(x^{(k)}_1\right)^\top,\left(x^{(k)}_2\right)^\top,\ldots, \left(x^{(k)}_N\right)^\top\right]^\top \in \mathbb{R}^{Nd}.
    \label{eq:concatenatedx}
\end{equation}
Define $G:\R^{Nd}\rightarrow \R$ with $G(x):=G(x_1,x_2,\cdots,x_N)=\sum_{i=1}^N\normsq{x_i-\proj(x_i)}$ and $F:\R^{Nd}\rightarrow \R$ with $F(x):=F(x_1,x_2\cdots,x_N)=\sum_{i=1}^N f_i(x_i)$. Then, we further define $U^{\gamma}(x):=F(x)+\frac{1}{2N\gamma}G(x)$ such that $U^{\gamma}(x)=\sum_{i=1}^{N}u_{i}^{\gamma}(x_{i})$, where
$u_{i}^{\gamma}(x_{i}):=f_{i}(x_{i})+\frac{q_{\mathcal{K}}^{\gamma}(x_{i})}{N}$. Let $x_{*}^\gamma\in \R^d$ be the unique minimizer of $u^\gamma$, and $x^*_\gamma=[(x_{*}^\gamma)^\top,\cdots,(x_{*}^\gamma)^\top]^\top$ is a vector in $\mathbb{R}^{Nd}$.
Then the concatenated DE-PSGLD iterations can be written as follows
\begin{equation}
    x^{(k+1)}=\Wc x^{(k)} -\eta \nabla U^{\gamma}\paren{x^{(k)}}-\eta \xi^{(k+1)}+\sqrt{2\eta}w^{(k+1)},
    \label{eq:concatenatedxk}
\end{equation}
where $\nabla U^{\gamma}(x)=\nabla F\paren{x}+\frac{1}{N\gamma}\nabla G(x)$, $\mathcal{W}=W\otimes \mb{I}_d,\,$ and $\otimes$ denotes the Kronecker product, and the noise terms are defined as 
\begin{align*}
    w^{(k+1)} &= \left[\left(w^{(k+1)}_1\right)^\top,\ldots, \left(w^{(k+1)}_N\right)^\top\right]^\top,\quad
    \xi^{(k+1)} = \left[\left(\xi^{(k+1)}_1\right)^\top,\ldots, \left(\xi^{(k+1)}_N\right)^\top\right]^\top,
\end{align*}
with the properties $\mathbb{E}\left[\xi^{(k+1)} \vrule\ \mathcal{F}_k\right]= 0$ and $\mathbb{E}\left\|\xi^{(k+1)}\right\|^2 \le \sigma^2N$ that follows from Assumption~\ref{assump:gradnoise}. Let $\bar{x}^{(k)}: = \frac{1}{N}\sum_{i=1}^{N}x_i^{(k)}$ denote the mean iterate at the $k$-th iteration. Since $\Wc$ is doubly stochastic, the average iterates satisfy:
\begin{align}
    \bar{x}^{(k+1)}&=\bar{x}^{(k)}-\eta\bsqb{\frac1N\sum_{i=1}^N\nabla f_i\paren{x_i^{(k)}}+\frac{1}{N\gamma}\paren{\bar{x}^{(k)}-\frac1N\sum_{i=1}^N\proj\paren{x_i^{(k)}}}}\nonumber\\
    &\qquad\qquad\qquad-\eta\bar{\xi}^{(k+1)}+\sqrt{2\eta}\bar{w}^{(k+1)},
    \label{eq:meanite}
\end{align}
where $\bar{\xi}^{(k+1)}:=\frac1N\sum_{i=1}^{N}\xi_i^{(k+1)}$, and $\bar{w}^{(k+1)}:=\frac1N\sum_{i=1}^{N} w_i^{(k+1)}$. 

Now we state the main result that provides the finite-time convergence guarantees. 

\begin{theorem}
    \label{thm:mt}
    Assume that $\E\normsq{x^{(0)}}<\infty$, and the stepsize $0<\eta<\eta_{\max}:=\min\left\{\frac{2N}{L_\gamma},\frac{1+\lambda_N^W}{L_\gamma},\frac{1}{L_\gamma+\mu}\right\}$ with $L_\gamma:=L+\frac{2}{N\gamma}$. Then, for every $k\in\mathbb{N}$, DE-PSGLD updates $x_i^{(k)}$ and their average $\bar{x}^{(k)}$ satisfy
    \begin{align}
        &\Wc_2\paren{\Law\paren{\bar{x}^{(k)}},\pi}\le
        (1-\mu\eta)^k\mc{C}_0+\paren{\frac{\rho^{2k}-\paren{1-\frac{\mu\eta}{N}\paren{1-\frac{\eta L_\gamma}{2N}}}^k}{\rho^2-\paren{1-\frac{\mu\eta}{N}\paren{1-\frac{\eta L_\gamma}{2N}}}}}^\frac12\frac{\rho}{\sqrt{N}}\mc{C}_1+\sqrt{\eta}\,\mc{C}_2\nonumber\\
        &\qquad\qquad\qquad\qquad\quad+\mc{C}\gamma^{1/8}\paren{\log\paren{1/\gamma}}^{1/8},\nonumber\\
        &\frac{1}{N}\sum_{i=1}^N\Wc_2\paren{\Law\paren{x_i^{(k)}},\pi}\le (1-\mu\eta)^k\mc{C}_0+\paren{\frac{\rho^{2k}-\paren{1-\frac{\mu\eta}{N}\paren{1-\frac{\eta L_\gamma}{2N}}}^k}{\rho^2-\paren{1-\frac{\mu\eta}{N}\paren{1-\frac{\eta L_\gamma}{2N}}}}}^\frac12\frac{\rho}{\sqrt{N}}\mc{C}_1\nonumber\\
        &\qquad\qquad\qquad\qquad\qquad\qquad\quad+\sqrt{\eta}\,\tilde{\mc{C}}_2+\frac{\rho^k}{\sqrt{N}}\mc{C}_3+\eta\,\mc{C}_4+\mc{C}\gamma^{1/8}\paren{\log\paren{1/\gamma}}^{1/8},\nonumber
    \end{align}
    where $\mc{C},\,\mc{C}_0,\, \mc{C}_1,\, \mc{C}_2,\,\tilde{\mc{C}}_2,\,\mc{C}_3,\mc{C}_4>0$ are explicit constants defined in the Appendix that depend on $\mu, L, \gamma, \rho,\sigma, d, N,\gamma_0$, and $\pi\propto e^{-f(x)}$ is the Gibbs distribution supported on $\mathcal{K}$.
\end{theorem}

Next, by using Theorem~\ref{thm:mt}, we obtain the iteration complexity in terms of the dependence on the accuracy level $\varepsilon$ and the dimension $d$.

\begin{corollary}
    \label{cor:revised}
    For any given $\varepsilon>0$ that is sufficiently small, we have $\Wc_2\paren{\Law\paren{\bar{x}^{K},\pi}}\le \tilde{\mc{O}}(\varepsilon)$, and $\frac1N\sum_{i=1}^N\Wc_2\paren{\Law\paren{x_i^{(K)},\pi}}\le \tilde{\mc{O}}(\varepsilon)$, where $\tilde{\mc{O}}$ hides the logarithmic dependence on $\varepsilon$, provided that $\gamma=\mc{O}(\varepsilon^8)$, $\eta=\mc{O}(\varepsilon^{18}/d)$ and $K=\mc{O}\paren{\frac{1}{\eta}\log\paren{\frac{\sqrt{d}}{\varepsilon^{9}}}}$. In particular, the iteration complexity is given by $K=\mc{O}\paren{\frac{d}{\varepsilon^{18}}\log\paren{\frac{\sqrt{d}}{\varepsilon^{9}}}}$.
\end{corollary}
The proof of Theorem~\ref{thm:mt} and Corollary~\ref{cor:revised} are given in Appendix~\ref{mth-proof} and Appendix~\ref{cor-proof}, respectively.
\section{Numerical Experiments}
\label{sec:numerics}

This section presents numerical experiments illustrating how the proposed \emph{Decentralized Proximal Stochastic Gradient Langevin Dynamics} (DE-PSGLD) can be used for sampling from constrained posteriors that arise in data science, machine learning, and sampling problems. In all of the experiments below, we use four types of network structures, such as fully connected, circular, star, and disconnected, as depicted in Figure~\ref{fig:all_networks}.

\begin{figure}[htbp]
    \centering
    \begin{minipage}[t]{0.24\textwidth}
        \centering
        \begin{tikzpicture}[baseline=(current bounding box.center)]
            \foreach \x in {1,...,6} {
                \node[draw, circle, fill=black!70, minimum size=5mm, inner sep=1.5pt] (node\x) at (360/6*\x:1.2) {};
            }
            \foreach \x in {1,...,6} {
                \foreach \y in {\x,...,6} {
                    \draw (node\x) -- (node\y);
                }
            }
        \end{tikzpicture}
        \caption*{\small{(a) Fully connected}}
    \end{minipage}
\hfill
    \begin{minipage}[t]{0.24\textwidth}
        \centering
        \begin{tikzpicture}[baseline=(current bounding box.center)]
            \foreach \x in {1,...,6} {
                \node[draw, circle, fill=black!70, minimum size=5mm, inner sep=1.5pt] (node\x) at (360/6*\x:1.2) {};
            }
            \foreach \x [remember=\x as \lastx (initially 6)] in {1,...,6} {
                \draw (node\lastx) -- (node\x);
            }
        \end{tikzpicture}
        \caption*{\small{(b) Circular}}
    \end{minipage}
\hfill
    \begin{minipage}[t]{0.24\textwidth}
        \centering
        \begin{tikzpicture}[baseline=(current bounding box.center)]
            \node[draw, circle, fill=black!70, minimum size=5mm, inner sep=1.5pt] (center) at (0,0) {};
            \foreach \x in {1,...,5} {
                \node[draw, circle, fill=black!70, minimum size=5mm, inner sep=1.5pt] (node\x) at (360/5*\x:1.2) {};
                \draw (center) -- (node\x);
            }
        \end{tikzpicture}
        \caption*{\small{(c) Star}}
    \end{minipage}
\hfill
    \begin{minipage}[t]{0.24\textwidth}
        \centering
        \begin{tikzpicture}[baseline=(current bounding box.center)]
            \foreach \x in {1,...,6} {
                \node[draw, circle, fill=black!70, minimum size=5mm, inner sep=1.5pt] (node\x) at (360/6*\x:1.2) {};
            }
        \end{tikzpicture}
        \caption*{\small{(d) Disconnected}}
    \end{minipage}
    \vspace{0.5cm}
    \caption{Different types of network structures}
    \label{fig:all_networks}
\end{figure}
\textbf{Synthetic 1-Dimensional Sampling.} We start our experiment by sampling from a non-Gaussian target on a compact convex set $\mk=[-1, 1]$ in one dimension, $d=1$, with density $\pi(x)\propto e^{-f(x)}$ where $f(x)=\frac{x^2}{2}+\frac{x^4}{8}-x$, and  $x\in\mk$. The target is smooth and log-concave near the origin. To evaluate sampling accuracy, we compute the 2-Wasserstein distance in 1 dimension~\citep{panaretos2019statistical} based on the discretization below:
\[
\mathcal{W}_2^2(\mu, \nu) = \int_0^1 \left(F^{-1}_{\mu}(u)-F^{-1}_{\nu}(u)\right)^2 du\approx \frac1n\sum_{k=1}^n \left(Q(u_k)-x^{(k)}\right),\quad u_k:=\frac{k-\frac12}{n},
\]
where $\mu$ denotes the true distribution with CDF $F$, 
$\nu$ denotes the distribution of the sampler, 
$Q=F^{-1}$ is the true quantile of the target on $[-1, 1]$, and $x^{(k)}$ are stored samples at the $k$-th iteration. We use the step size $\eta=5\times 10^{-4}$ and proximal regularizer $\gamma=3.3\times 10^{-4}$, and we iterate through $300$ steps. In each iteration, we generate $100$ samples of the variable $x$ for both DE-PSGLD and PSGLD algorithms, and for decentralization, we used a $30$-agent mixing matrix. 

\begin{figure}
    \centering
    \includegraphics[width=\linewidth]{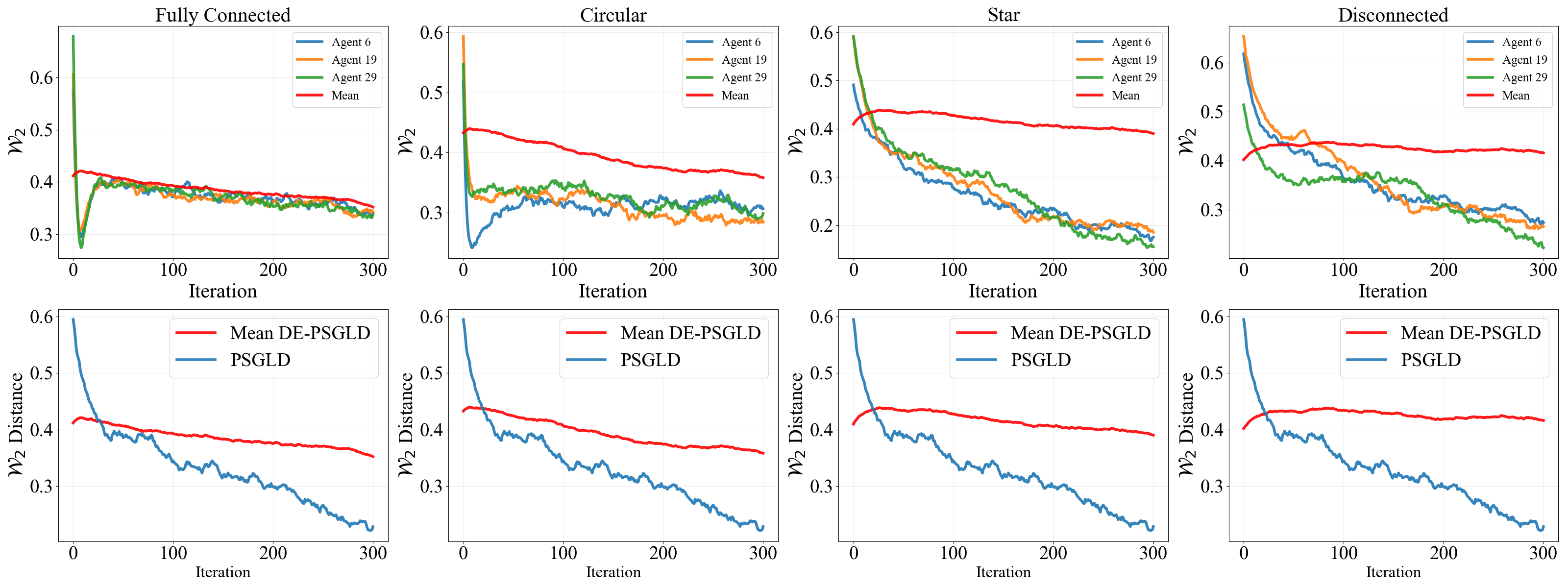}
    \caption{Evolution of the 2-Wasserstein distance between the target distribution and the samples generated by decentralized proximal SGLD (DE-PSGLD) and centralized Moreau-Yosida SGLD (PSGLD) across different network styles.}
    \label{fig:1dsampling}
\end{figure}

The top row in Figure~\ref{fig:1dsampling} shows the 2-Wasserstein distances between the true distribution and randomly chosen three representative agents and the network mean across four networks. These plots suggest that a stronger connection yields rapid agent agreement and a steady decay of the 2-Wasserstein distance. However, a weaker connection, such as a star network, yields slow convergence. From the bottom row, we see that the centralized PSGLD converges faster than the mean DE-PSGLD due to the absence of communication constraints. The DE-PSGLD samples concentrate most of their mass inside the feasible region, and the mean chain is more stable than a single agent. In contrast, centralized PSGLD exhibits a wider distribution with more mass outside $[-1,1]$ as shown in Figure~\ref{fig:1D_density_compare}.

\begin{figure}
    \centering
    \includegraphics[width=\linewidth]{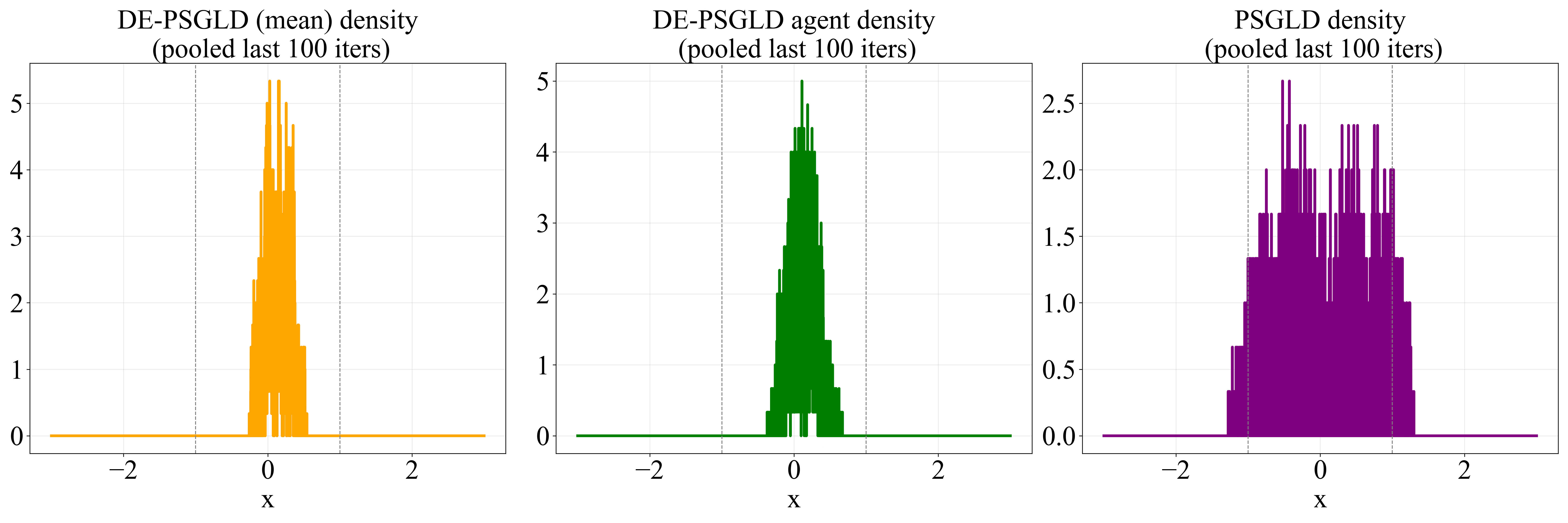}
    \caption{Comparison of the sampling from the target density $\pi(x)\propto e^{-f(x)}\mb{1}_{[-1,1]}$ with empirical densities obtained from DE-PSGLD and PSLGD. The mean and agent 1 are from a fully connected network style.}
    \label{fig:1D_density_compare}
\end{figure}

\textbf{Synthetic 2-Dimension: Bayesian Linear Regression.} The constrained sampling has great usefulness in the Bayesian regression problem. For instance, if we set the constraint that the model parameters are bounded in the $\ell_1$-ball, then the model is referred to as the Lasso regression; if we set the constraint to be the $\ell_2$-ball, then the regression is referred to as the ridge regression. As a toy experiment, we consider a two-dimensional problem with synthetic data generated using the following model
\begin{equation}
    y_i = \beta_*^{\top}X_i + \epsilon_i,\qquad \epsilon_i\sim\mathcal{N}(0, 0.25),\qquad X_i\sim\mathcal{N}(0, \mathbf{I}_2),\qquad \beta_* = [1, 1]^{\top},\label{eq:modelblr}    
\end{equation}
where $\beta_*$ is the true value of the parameters. Given the data, our goal is to sample from the posterior distribution $p(\beta\mid \mb{X, y}) \propto p(\mb{X}\mid \beta)p(\beta)$. The prior $p(\beta)$ is defined as a uniform distribution over the $\ell_p$-ball of radius $s$. Specifically, we set $s$ to be $80\%$ of the norm of the Ordinary Least Squares (OLS) estimate, i.e., $\mk = \{\beta: \norm{\beta}_2\le s\}$ and $s=\frac{4}{5}\norm{\beta_{\mathrm{OLS}}}_{2}$. It is uniformly distributed in the $\ell_2$-ball centered at the origin as shown in Figure~\ref{fig:2dprior}. To run the sampling, we simulate 10,000 observations using the model~\eqref{eq:modelblr} and distribute them among $N=20$ agents. Under the assumption of Gaussian noise $\epsilon_i\sim \mathcal{N}(0, \sigma^2)$, the log-likelihood corresponds to the standard least-squares loss $f(\beta) = \frac{1}{2\sigma^2}\sum_{i=1}^{N}\sum_{j=1}^{n_i} (y_j-\beta^{\top}X_j)^2$, where $n_i=500$ is the number of data points in agent $i$.
\begin{figure}[htbp]
    \centering
    \begin{minipage}{0.55\textwidth}
        For this set of experiments, we take a batch size of 100 and run both DE-PSGLD and PSGLD algorithms for 500 iterations with a fixed proximal regularizer $\gamma=5\times 10^{-5}$ and stepsize $\eta=5\times 10^{-4}$. At each iteration, we generate $300$ samples of the parameters.
        
        The results from the PSGLD are shown in Figure~\ref{fig:mysgld2dpost}: the PSGLD algorithm successfully concentrates the posterior density near the boundary of the $\ell_2$-ball closest to the true parameter $\beta_*$. Since the true parameter (the red star) lies outside the constraint set $\mathcal{K}$, the posterior density is ``pushed'' against the boundary, illustrating how the proximal term $\frac{1}{2\gamma}\|\beta - \proj(\beta)\|^2$ allows the sampler to respect the hard constraint while targeting the likelihood. 

        The result from the DE-PSGLD algorithm is shown in Figure~\ref{fig:blrsampldpsgld}. We observe the impact of the network structure on the sampling quality. Networks with some form of connections, such as fully connected, circular, and star, show strong consensus. The local posterior is nearly identical to the ensemble mean (bottom row). The fourth column, \emph{i.e.,} the disconnected network sets the baseline in sampling. Since the agents are not connected and do not communicate, the samples are nearly uniformly distributed within the constrained region, and the mean is far from the true parameters. It shows the importance of communication and decentralization.
    \end{minipage}
    \hfill 
    \begin{minipage}{0.4\textwidth}
        \centering
        \includegraphics[scale=0.22]{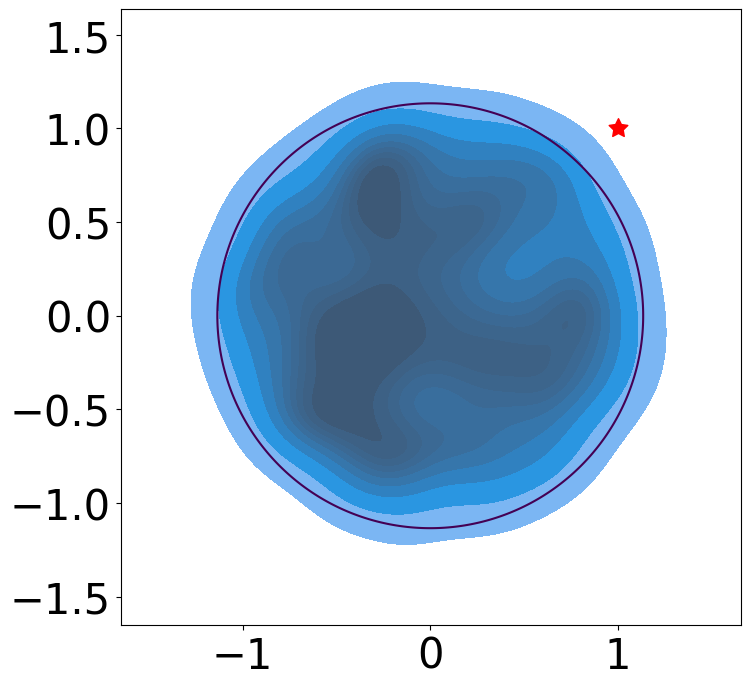}
        \caption{The contour plot prior distribution constrained in the 2-norm ball around the origin.}
        \label{fig:2dprior}
        \includegraphics[scale=0.22]{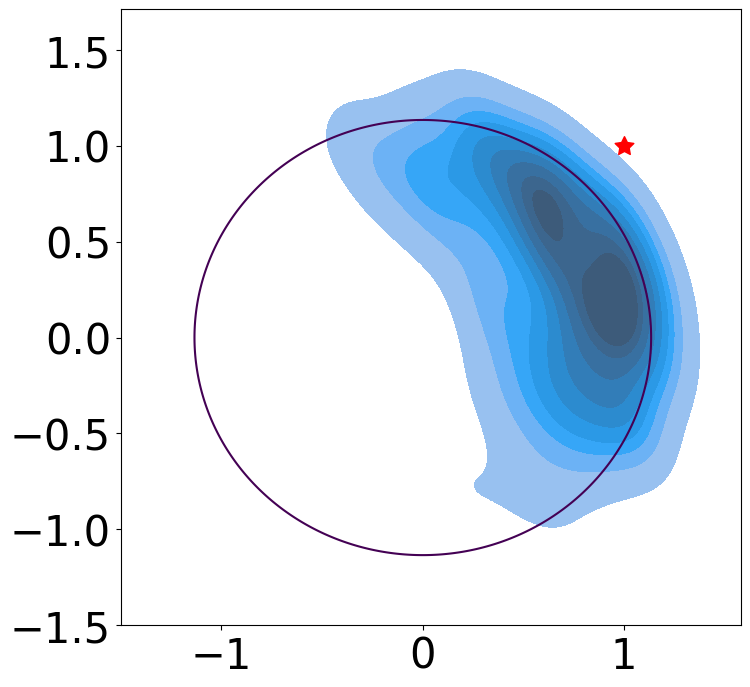}
        \caption{The contour plot of the posterior distribution of the model parameter sampled using the PSGLD algorithm.}
        \label{fig:mysgld2dpost}
    \end{minipage}
\end{figure}
\begin{figure}
    \centering
    \includegraphics[width=\linewidth]{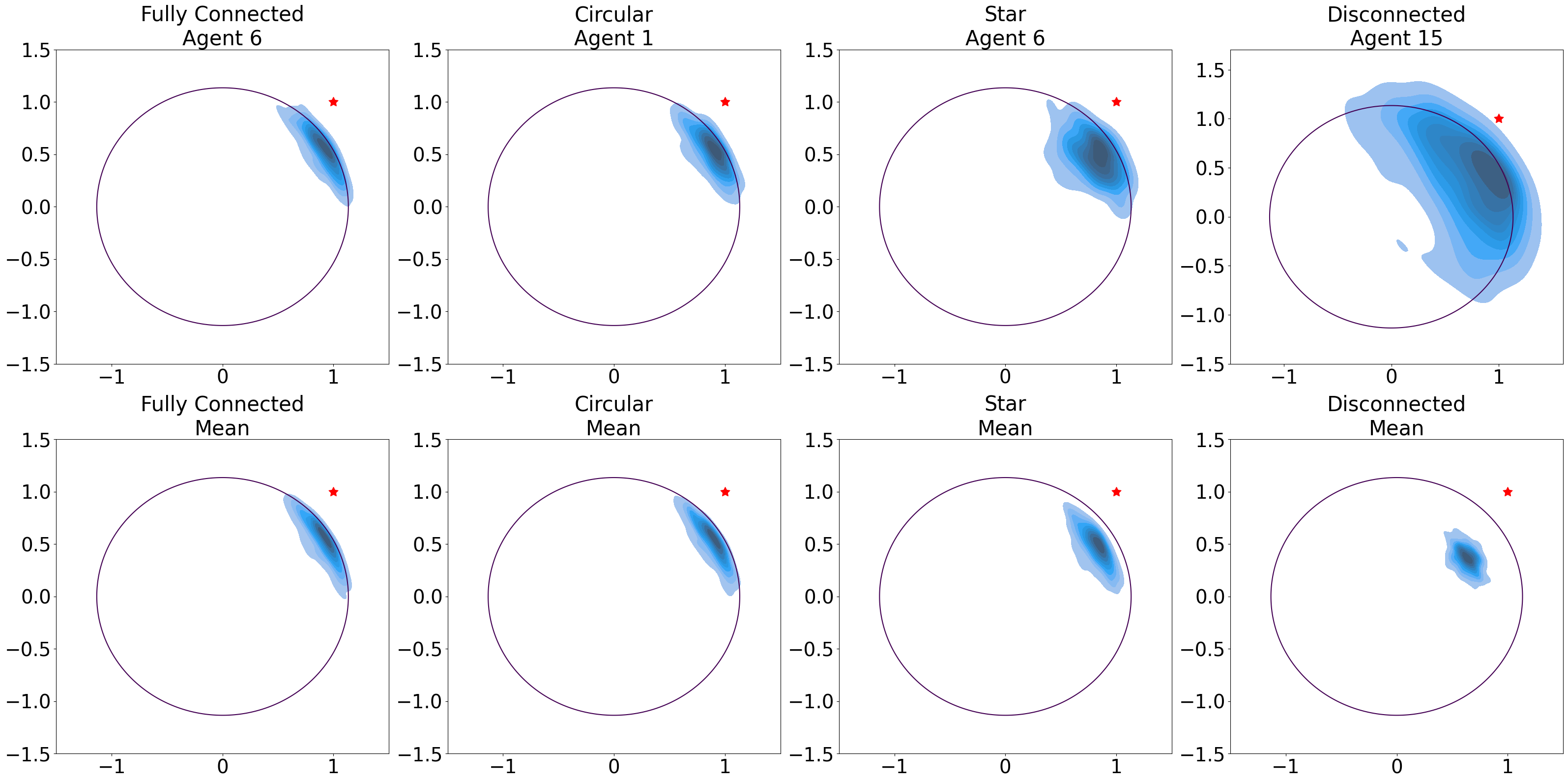}
    \caption{The top row shows the contour plot of samples of a randomly chosen agent out of 20 agents. The bottom row shows the mean of all agents across all 4 network structures. The red star is the true location of the parameter in 2D space.}
    \label{fig:blrsampldpsgld}
\end{figure}\hfill\\
\textbf{Real $d$-Dimension: Bayesian Logistic Regression.} In this set of experiments, we use UCI ML Breast Cancer Wisconsin (Diagnostic) datasets~\citep{breast}. The data is in the form of an input-output pair $(X_i,y_i)_{i=1}^{569}$ where each $X_i\in \R^{30}$ represents a feature vector containing 30 predictive attributes, and the binary target $y_i\in \{0,1\}$ represents the diagnostic classification: $212$ malignant vs $357$ benign. The prior $p(\beta)$ is taken from the $\ell_2$-ball of radius $s$, where $s$ is the $80\%$ of the norm of the $\beta$ parameters from unconstrained logistic regression using the Maximum Likelihood Estimation (MLE) method. In our experiment, we consider a network of 5 agents and distribute the data equally among the agents. Figure~\ref{fig:bcan} illustrates the progress in classification accuracy as we increase the number of iterations up to $1000$. Each iteration generates $1000$ samples, and uses a mini-batch size of $10$, stepsize $\eta=0.005$, proximal parameter $\gamma = 0.16$. 
\begin{figure}[h]
    \centering
    \includegraphics[width=\linewidth]{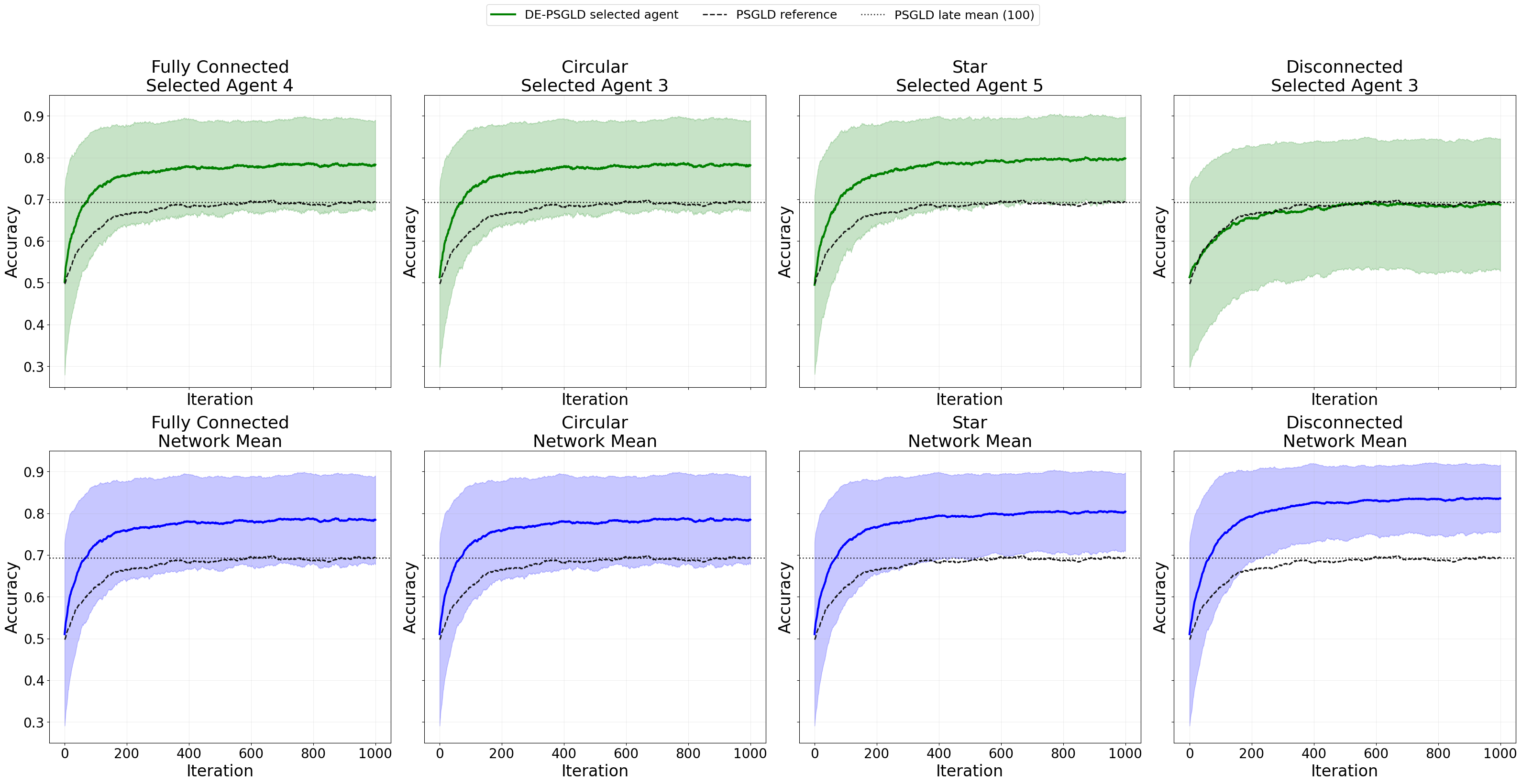}
    \caption{The plots show the accuracy versus iterations using the DE-PSGLD and PSGLD algorithms in the classification of cancer types. The top row shows the relative performances of a randomly selected agent from the DE-PSGLD algorithm and the PSGLD algorithm across four different network structures. The bottom row displays the mean accuracy over all agents from the DE-PSGLD algorithm versus the PSGLD algorithm across four network structures.}
    \label{fig:bcan}
\end{figure} 

\section{Conclusion}
In this paper, we studied the problem of constrained sampling in a distributed system where the agents learn to sample from a log-concave target distribution $\pi(x)\propto e^{-f(x)}$ on a constrained domain $\mathcal{K}\subset\mathbb{R}^{d}$ without sharing the local data. We proposed and studied the decentralized proximal stochastic gradient Langevin dynamics (DE-PSGLD). We provided a non-asymptotic 2-Wasserstein convergence guarantee and iteration complexity for a given accuracy level $\varepsilon$ when the target $f$ is smooth and strongly convex, and follow some assumptions. Finally, we provide numerical experiments to show the efficiency of our algorithm for sampling in a decentralized learning system.

\section*{Acknowledgments}
The authors are grateful to Mert G\"{u}rb\"{u}zbalaban for helpful discussions.
Mohammad Rafiqul Islam is partially supported by the grant NSF DMS-2053454. 
Lingjiong Zhu is partially supported by the grants NSF DMS-2053454 and DMS-2208303. 

\newpage
{\small 
\bibliographystyle{plainnat}
\bibliography{dpsgld}}


\newpage
\appendix

\section{Proof of Theorem~\ref{thm:mt}}
To prove Theorem~\ref{thm:mt}, we split the arguments into the following steps:
\label{mth-proof}
\begin{enumerate}
    \item First, we compute an upper bound on the $L^2$ distance between the iterate $x_i^{(k)}$ and their average  iterate $\bar{x}^{(k)}$; 
    \item Next, we compute an upper bound on the $L^2$ distance between the average iterate $\bar{x}^{(k)}$ and the iterate of the centralized proximal Langevin algorithm (PLA), also known as Moreau--Yosida Unadjusted Langevin Algorithm (MYULA) in the literature~\citep{SR2020};
    \item Finally, we compute an upper bound on the $\displaystyle \Wc_2$ distance between the law of $x^{(k)}$ and the Gibbs distribution $\pi$. 
\end{enumerate}

The whole process breaks down into the following triangle inequality
\begin{equation}
    \label{thm:proof:eq1}
    \frac{1}{N}\sum_{i=1}^N\Wc_2\paren{\Law\paren{x_i^{(k)}},\pi}\le \frac1N\sum_{i=1}^N\Wc\paren{\Law\paren{x_i^{(k)}},\Law\paren{\bar{x}^{(k)}}}+\Wc\paren{\Law\paren{\bar{x}^{(k)}},\pi},
\end{equation}
where
\begin{equation}
\Wc_2\paren{\Law\paren{\bar{x}^{(k)}},\pi} \le \Wc_2\paren{\Law\paren{\bar{x}^{(k)}},\Law\paren{\tilde{x}^{(k)}}}+\Wc_2\paren{\Law\paren{\tilde{x}^{(k)}},\pi^{\gamma}}+\Wc_2\paren{\pi^\gamma,\pi}.
\label{eq:avW2}
\end{equation}
Here, equation~\eqref{thm:proof:eq1} represents the average of the 2-Wasserstein distance between the distribution of the iterates of each agent and the target distribution, and equation~\eqref{eq:avW2} represents the 2-Wasserstein distance between the distribution of the average iterates from all agents and the target distribution. The step-by-step proof is given in the subsequent sections.

\subsection{\textbf{Uniform $L^2$ bounds between $x_i^{(k)}$ and their average $\bar{x}^{k}$}}

We first start with computing a uniform $L^2$ bound on the gradient $\nabla U^{\gamma}\paren{x^{(k)}}$ using the following lemma. The idea is to consider the concatenated DE-PSGLD iterates in equation~\eqref{eq:concatenatedxk} as the decentralized gradient descent (DGD) subject to the proximal regularizer, stochastic gradient, and Gaussian noise. 
\begin{lemma}
    \label{lem:gradbound}
    Under the assumptions of Theorem~{\ref{thm:mt}},
    \begin{equation*}
        \E\normsq{\nabla U^{\gamma}\paren{x^{(k)}}} \le D_\gamma^2,\qquad\textrm{for any }k\in\mathbb{N},
    \end{equation*}
    where 
    \begin{equation}
        \begin{split}
            D_\gamma^2&:=4L_\gamma^2\paren{1-\mu\eta\paren{1+\lambda_N^W-\eta L_\gamma}}^k\normsq{x^{(0)}-x^*_{\gamma}} + 8L_\gamma^2\frac{C_1^2\eta^2 N}{(1-\rho)^2}\\
         &\qquad\qquad+4\normsq{\nabla U^{\gamma}(x^*_\gamma)}+\frac{2L_\gamma^2(\eta \sigma^2 N+2dN)}{\mu\paren{1+\lambda_N^W-\eta L_\gamma}},
        \end{split}
        \label{eq:D}
    \end{equation}
    with $L_\gamma := L+\frac{2}{N\gamma}$, and the constant $C_1$ is defined in equation~\eqref{eq:c1}, $1-\rho$ is the spectral gap, $x^*_\gamma=\left[\left(x_{*}^\gamma\right)^\top,\cdots,\left(x_{*}^\gamma\right)^\top\right]^\top\in \R^{Nd}$ with $x_*^\gamma\in \R^d$ being the unique minimizer of $u^\gamma$.
\end{lemma}

The proof of Lemma~\ref{lem:gradbound} is given in Appendix~\ref{lemproof:gradbound}. This bound provides a useful tool to bound the uniform $L_2$ bound between the iterates $x_i^{(k)}$ and the mean iterates $\bar{x}^{(k)}$, which is given in the following lemma
\begin{lemma}
    \label{lem:agent-avg}
    Under the assumptions of Theorem~\ref{thm:mt},
    \begin{align*}
        \frac1N\sum_{i=1}^N\E\normsq{x_i^{(k)}-\bar{x}^{(k)}}\le \frac{4\rho^{2k}}{N}\E\normsq{x^{(0)}}+ \frac{4\eta^2D_\gamma^2}{N(1-\rho)^2}+\frac{4\eta^2\sigma^2}{(1-\rho^2)}+\frac{8\eta d }{(1-\rho^2)}.
    \end{align*}
\end{lemma}

The proof of Lemma~\ref{lem:agent-avg} is given in Appendix~\ref{lemproof:agent-avg}. 

We use the Cauchy-Schwarz inequality in Lemma~\ref{lem:agent-avg} to upper bound the first component of~{\ref{thm:proof:eq1}}, 
\begin{align}  &\frac1N\sum_{i=1}^N\Wc_2\paren{\Law\paren{x_i^{(k)}},\Law\paren{\bar{x}^{(k)}}}\nonumber
    \\   &\le\paren{\frac1N\sum_{i=1}^N\Wc_2^2\paren{\Law\paren{x_i^{(k)}},\Law\paren{\bar{x}^{(k)}}}}^\frac12\nonumber\\
    &\le \paren{\frac1N\sum_{i=1}^N\E\normsq{x_i^{(k)}-\bar{x}^{(k)}}}^\frac12\nonumber\\
    &\le \paren{\frac{4\rho^{2k}}{N}\E\normsq{x^{(0)}}+ \frac{4\eta^2D_\gamma^2}{N(1-\rho)^2}+\frac{4\eta^2\sigma^2}{(1-\rho^2)}+\frac{8\eta d }{(1-\rho^2)}}^\frac12\nonumber\\
    &\le \frac{2\rho^k}{\sqrt{N}}\paren{\E\normsq{x^{(0)}}}^\frac12+\frac{2\eta D_\gamma}{(1-\rho)\sqrt{N}}+\frac{2\eta\sigma }{\sqrt{1-\rho^2}}+\frac{2\sqrt{2\eta\,d}}{\sqrt{1-\rho^2}}.\label{thm:proof:eq2}
\end{align}

\subsection{\textbf{$L^2$ Distance between the mean $\bar{x}^{(k)}$ and the centralized proximal Langevin (PLA)} $\tilde{x}^{(k)}$}
When the data is distributed among the agents, the learning process introduces errors between the average of the gradients and the gradient of the average across all agents. Let us define the error term at the $(k+1)$-th iteration as follows:
\begin{align}
    \mc{E}_{k+1}:=\frac1N\sum_{i=1}^N\bsqb{\nabla f_i\paren{\bar{x}^{(k)}}-\nabla f_i\paren{x_i^{(k)}}} +\frac{1}{N^{2}\gamma}\sum_{i=1}^N\bsqb{\proj\paren{\bar{x}^{(k)}}-\proj\paren{x_i^{(k)}}}.   
    \label{eq:err}
\end{align}
Since the gradients $\nabla f_i$'s are Lipschitz based on our Assumption~\ref{assumption:convex} and the projection is 1-Lipschitz, we can find an $L^{2}$ bound for the error defined in~\eqref{eq:err} presented in the following lemma.

\begin{lemma}
    \label{lem:errbound}
    Under the assumptions of Theorem~\ref{thm:mt}, for any $k\in\mathbb{N}$, we have
    \begin{equation*}
        \E \normsq{\mc{E}_{k+1}}\le \frac{8L_\gamma^2\rho^{2k}}{N}\E\normsq{x^{(0)}}+\frac{8L_\gamma^2\eta^2D_\gamma^2}{N(1-\rho)^2}+\frac{8L_\gamma^2\eta^2\sigma^2}{(1-\rho^2)}+\frac{16L_\gamma^2\eta d }{(1-\rho^2)}.
    \end{equation*}
\end{lemma}

The proof of Lemma~\ref{lem:errbound} is given in Appendix~\ref{lemproof:errorbound}. 
Next, let us define the mean-chain potential
\[
\hat{u}^\gamma(x):=\frac1Nf(x)+\frac1Nq_{\mk}^\gamma(x),
\]
such that
\[
\nabla \hat{u}^\gamma(x)=\frac1N\nabla f(x)+\frac{1}{N\gamma}\paren{x-\proj(x)}.
\]
Therefore, the mean iterates from~\eqref{eq:meanite} can be rewritten as follows:
\begin{align}
    \bar{x}^{(k+1)} &= \bar{x}^{(k)}-\eta\nabla\hat{u}^\gamma\paren{\bar{x}^{(k)}}+\eta\mc{E}_{k+1}-\eta\bar{\xi}^{(k+1)}+\sqrt{2\eta}\bar{w}^{(k+1)},
    \label{eq:meanite-again}
\end{align}
where $\mc{E}_{k+1}$ is defined in~\eqref{eq:err}, $\bar{\xi}^{(k+1)}:=\frac1N\sum_{i=1}^{N} \xi_i^{(k+1)}$ with $\xi_{i}^{(k+1)}$ being the gradient noise satisfying Assumption~\ref{assump:gradnoise}, and $\bar{w}^{(k+1)}:=\frac1N\sum_{i=1}^{N} w_i^{(k+1)}\sim\mathcal{N}(0,N^{-1}\mb{I}_d)$. Now, let us define the centralized proximal Langevin algorithm (PLA) by
\begin{equation}
    \label{eq:myulac}
    \tilde{x}^{(k+1)}=\tilde{x}^{(k)}-\eta\nabla\hat{u}^\gamma\paren{\tilde{x}^{(k)}}+\sqrt{2N^{-1}\eta}\,z^{(k+1)},
\end{equation}
where $\displaystyle z^{(k+1)}\sim \mathcal{N}\paren{\mb{0},\mb{I}_d}$, and coupled noise $\displaystyle z^{(k+1)}=\sqrt{N}\bar{w}^{(k+1)}$.
This is the Euler-Maruyama discretization of the continuous-time overdamped Langevin diffusion
\begin{equation}
    dX_t:=-\nabla \hat{u}^\gamma(X_{t})dt+\sqrt{2N^{-1}}dW_t,\label{eq:ovld}
\end{equation}
where $W_t$ is a standard $d$-dimensional Brownian motion.
We are interested in upper-bounding the $L^2$ distance between the mean iterate $\bar{x}^{(k)}$ and the centralized proximal Langevin algorithm (PLA) iterate $\tilde{x}^{(k)}$. We have the following lemma.

\begin{lemma}
    \label{lem:myula-avg-bound}
    Under the assumptions of Theorem~\ref{thm:mt}, for every $k\in\mathbb{N}$,
    \begin{align*}
        &\E\normsq{\bar{x}^{(k)}-\tilde{x}^{(k)}}\nonumber\\
        &\le \eta \paren{\frac{\eta}{\frac{\mu}{N}\paren{1-\frac{\eta L_\gamma}{2N}}}+\frac{(1+\frac{\eta L_\gamma}{N})^2}{\frac{\mu^2}{N^{2}}\paren{1-\frac{\eta L_\gamma}{2N}}^2}}\paren{\frac{8L_\gamma^2\eta D_\gamma^2}{N(1-\rho)^2}+\frac{8L_\gamma^2\eta\sigma^2}{(1-\rho^2)}+\frac{16L_\gamma^2  d }{(1-\rho^2)}}\\
        &\qquad\qquad+\frac{\eta\sigma^2}{\mu\paren{1-\frac{\eta L_\gamma}{2N}}}+\frac{\rho^{2k}-\paren{1-\frac{\mu\eta}{N}\paren{1-\frac{\eta L_\gamma}{2N}}}^k}{\rho^2-\paren{1-\frac{\mu\eta}{N}\paren{1-\frac{\eta L_\gamma}{2N}}}}\cdot\frac{8L_\gamma^2\rho^2}{N}\E\normsq{x^{(0)}}.
    \end{align*}
\end{lemma}

The proof of Lemma~\ref{lem:myula-avg-bound} is given in Appendix~\ref{lemproof:myula-avg-bound}.

\subsection{\textbf{$\displaystyle \Wc_2$ Distance between the centralized PLA and the Gibbs distribution $\pi^\gamma$}}
The bounds in terms of the $\Wc_2$ distance between the law of the Euler-Maruyama discretization of $\tilde{x}^{(k)}$ from the centralized Moreau--Yosida regularized Langevin dynamics and the Gibbs distribution can be derived using the results from \citet{dalalyan2019user} with some adjustment in smoothness and convexity parameters.
Recall the mean-chain potential
\[
\hat{u}^\gamma(x):=\frac1Nf(x)+\frac1Nq_{\mk}^\gamma(x).
\]
Note that $\hat{u}^\gamma(x)$ is $\mu/N$-strongly convex and $L_\gamma/N$-smooth.
Therefore, based on Theorem~4 from \citet{dalalyan2019user}, we have the following lemma.

\begin{lemma}
    \label{lem:dalalyan}
    For any $ \eta\in (0,\frac{2N}{L_\gamma+\mu}]$, we have
    \begin{equation*}
        \Wc_2\paren{\Law\paren{\tilde{x}^{(k)}},\pi^\gamma}\le (1-\mu\eta)^k\Wc_2\paren{\Law\paren{\tilde{x}^{(0)}},\pi^\gamma}+\frac{1.65L_\gamma}{N\mu}\sqrt{\eta d}.
    \end{equation*}
\end{lemma}

Next, we need to bound the distance between the minimizer of $\hat{u}^{\gamma}(x)$ and the Gibbs distribution $\pi^{\gamma}$.

\begin{lemma}
    \label{lem:minimizer}
    Let $\hat{x}_\gamma^*$ be the unique minimizer of $\hat{u}^\gamma(x)$.
    Then we have
    \begin{equation}
        \E_{X\sim \pi^\gamma}\normsq{X-\hat{x}_\gamma^*}\le \frac{2d}{\mu}.\label{eq:minimizer}    
    \end{equation}
\end{lemma}

The proof of Lemma~\ref{lem:minimizer} is given in Appendix~\ref{lemproof:minimizer}. Setting $\displaystyle \tilde{x}^{(0)}=\frac1N\sum_{i=0}^{N}x_i^{(0)}$ and using~\eqref{eq:minimizer}, we have
\begin{align*}
\mathcal{W}_2\left(\Law\left(\tilde{x}^{(0)}\right),\pi^\gamma\right)&\le \paren{\E\normsq{\tilde{x}^{(0)}-\hat x^*_\gamma}}^{1/2}+\paren{\E_{X\sim\pi^{\gamma}}\normsq{X-\hat x^*_\gamma}}^{1/2}\\
    &\le \paren{\E\normsq{\tilde{x}^{(0)}-\hat x^*_\gamma}}^{1/2}+\sqrt{2d\mu^{-1}}.
\end{align*}
Therefore, it follows from Lemma~\ref{lem:dalalyan} that
\begin{equation}    \Wc_2\paren{\Law\paren{\tilde{x}^{(k)}},\pi^\gamma}\le (1-\mu\eta)^k\paren{\paren{\E\normsq{\tilde{x}^{(0)}-\hat x^*_\gamma}}^{1/2}+\sqrt{2d\mu^{-1}}}+\frac{1.65L_\gamma}{\mu}\sqrt{\eta d}.\label{eq:final2}
\end{equation}


\subsection{Proof of Theorem~\ref{thm:mt}}
To complete the proof, we follow from Lemma~\ref{lem:myula-avg-bound} and obtain that
\begin{align}   &\Wc_2\paren{\Law\paren{\bar{x}^{(k)}},\Law\paren{\tilde{x}^{(k)}}}\nonumber\\
    &\le \paren{\E\normsq{\bar{x}^{(k)}-\tilde{x}^{(k)}}}^\frac12\nonumber\\
    &\le \sqrt{\eta} \paren{\frac{\eta}{\frac{\mu}{N}\paren{1-\frac{\eta L_\gamma}{2N}}}+\frac{(1+\frac{\eta L_\gamma}{N})^2}{\frac{\mu^2}{N^{2}}\paren{1-\frac{\eta L_\gamma}{2N}}^2}}^\frac12\cdot\paren{\frac{8L_\gamma^2\eta D_\gamma^2}{N(1-\rho)^2}+\frac{8L_\gamma^2\eta\sigma^2}{(1-\rho^2)}+\frac{16L_\gamma^2  d }{(1-\rho^2)}}^\frac12\nonumber\\
    &\qquad+\frac{\sqrt{\eta}\sigma}{\sqrt{\mu\paren{1-\frac{\eta L_\gamma}{2N}}}}+\paren{\frac{\rho^{2k}-\paren{1-\frac{\mu\eta}{N}\paren{1-\frac{\eta L_\gamma}{2N}}}^k}{\rho^2-\paren{1-\frac{\mu\eta}{N}\paren{1-\frac{\eta L_\gamma}{2N}}}}}^\frac12\cdot\frac{2\sqrt{2}\rho L_\gamma}{\sqrt{N}}\paren{\E\normsq{x^{(0)}}}^\frac12.\label{eq:final1}
\end{align}
Now we provide the final form of~\eqref{eq:avW2} using the results from~\eqref{eneq:dist}, \eqref{eq:final2}, and \eqref{eq:final1} as follows:
\begin{align}
&\Wc_2\paren{\Law\paren{\bar{x}^{(k)}},\pi} \nonumber\\
    &\le \sqrt{\eta} \paren{\frac{\eta}{\frac{\mu}{N}\paren{1-\frac{\eta L_\gamma}{2N}}}+\frac{(1+\frac{\eta L_\gamma}{N})^2}{\frac{\mu^2}{N^{2}}\paren{1-\frac{\eta L_\gamma}{2N}}^2}}^\frac12\cdot\paren{\frac{8L_\gamma^2\eta D_\gamma^2}{N(1-\rho)^2}+\frac{8L_\gamma^2\eta\sigma^2}{(1-\rho^2)}+\frac{16L_\gamma^2  d }{(1-\rho^2)}}^\frac12\nonumber\\
    &\qquad+\frac{\sqrt{\eta}\sigma}{\sqrt{\mu\paren{1-\frac{\eta L_\gamma}{2N}}}}+\paren{\frac{\rho^{2k}-\paren{1-\frac{\mu\eta}{N}\paren{1-\frac{\eta L_\gamma}{2N}}}^k}{\rho^2-\paren{1-\frac{\mu\eta}{N}\paren{1-\frac{\eta L_\gamma}{2N}}}}}^\frac12\cdot\frac{2\sqrt{2}\rho L_\gamma}{\sqrt{N}}\paren{\E\normsq{x^{(0)}}}^\frac12\nonumber\\
    &\qquad+(1-\mu\eta)^k\paren{\paren{\E\normsq{\tilde{x}^{(0)}-\hat x^*_\gamma}}^{1/2}+\sqrt{2d\mu^{-1}}}
+\frac{1.65L_\gamma}{\mu}\sqrt{\eta d}+\mathcal{C}\gamma^{\frac18}\paren{\log\paren{\frac1\gamma}}^{\frac18}\nonumber\\
    &=(1-\mu\eta)^k\mc{C}_0+\paren{\frac{\rho^{2k}-\paren{1-\frac{\mu\eta}{N}\paren{1-\frac{\eta L_\gamma}{2N}}}^k}{\rho^2-\paren{1-\frac{\mu\eta}{N}\paren{1-\frac{\eta L_\gamma}{2N}}}}}^\frac12\frac{\rho }{\sqrt{N}}\mc{C}_1
+\sqrt{\eta}\mc{C}_2+\mathcal{C}\gamma^{1/8}\paren{\log\paren{1/\gamma}}^{1/8},\label{eq:final3}
\end{align}
with 
\begin{align*}
\mc{C}_0:=\paren{\paren{\E\normsq{\tilde{x}^{(0)}-\hat x^*_\gamma}}^{1/2}+\sqrt{2d\mu^{-1}}},
\qquad
\mc{C}_1:=2\sqrt{2}L_\gamma\paren{\E\normsq{x^{(0)}}}^\frac12,
\end{align*}
and
\begin{align*}
\mc{C}_2:&=\frac{1.65L_\gamma}{\mu}\sqrt{ d}+\frac{\sigma}{\sqrt{\mu\paren{1-\frac{\eta L_\gamma}{2N}}}}\\
    &\qquad+\paren{\frac{\eta}{\frac{\mu}{N}\paren{1-\frac{\eta L_\gamma}{2N}}}+\frac{(1+\frac{\eta L_\gamma}{N})^2}{\frac{\mu^2}{N^{2}}\paren{1-\frac{\eta L_\gamma}{2N}}^2}}^\frac12\cdot\paren{\frac{8L_\gamma^2\eta D_\gamma^2}{N(1-\rho)^2}+\frac{8L_\gamma^2\eta\sigma^2}{(1-\rho^2)}+\frac{16L_\gamma^2  d }{(1-\rho^2)}}^\frac12,
\end{align*}
where $\hat x^*_\gamma$ is the unique minimizer of $\hat{u}^\gamma$, $\tilde{x}^{(0)}=\frac1N\sum_{i=0}^{N}x_i^{(0)}$, $D_\gamma$ is defined in~\eqref{eq:D}, $\Law\paren{x_i^{(k)}}$ denotes the distribution of $x_i^{(k)}$, and $\pi$ is the Gibbs distribution with probability density function proportional to $e^{-f(x)}$. To obtain the final form of~\eqref{thm:proof:eq1}, we use~\eqref{thm:proof:eq2} and \eqref{eq:final3}, such that we can compute that
\begin{align}
    &\frac{1}{N}\sum_{i=1}^N\Wc_2\paren{\Law\paren{x_i^{(k)}},\pi}\nonumber\\
    &\le \frac{2\rho^k}{\sqrt{N}}\paren{\E\normsq{x^{(0)}}}^\frac12+\frac{2\eta D_\gamma}{(1-\rho)\sqrt{N}}+\frac{2\eta\sigma }{\sqrt{1-\rho^2}}+\frac{2\sqrt{2\eta\,d}}{\sqrt{1-\rho^2}}\nonumber\\
    &+(1-\mu\eta)^k\mc{C}_0+\paren{\frac{\rho^{2k}-\paren{1-\frac{\mu\eta}{N}\paren{1-\frac{\eta L_\gamma}{2N}}}^k}{\rho^2-\paren{1-\frac{\mu\eta}{N}\paren{1-\frac{\eta L_\gamma}{2N}}}}}^\frac12\frac{\rho }{\sqrt{N}}\mc{C}_1+\sqrt{\eta}\mc{C}_2+\mathcal{C}\gamma^{1/8}\paren{\log\paren{1/\gamma}}^{1/8}\nonumber\\
    &=(1-\mu\eta)^k\mc{C}_0+\paren{\frac{\rho^{2k}-\paren{1-\frac{\mu\eta}{N}\paren{1-\frac{\eta L_\gamma}{2N}}}^k}{\rho^2-\paren{1-\frac{\mu\eta}{N}\paren{1-\frac{\eta L_\gamma}{2N}}}}}^\frac12\frac{\rho }{\sqrt{N}}\mc{C}_1\nonumber\\    &\qquad\qquad\qquad\qquad+\sqrt{\eta}\tilde{\mc{C}_2}+\frac{\rho^k}{\sqrt{N}}\mc{C}_3+\eta\mc{C}_4+\mathcal{C}\gamma^{1/8}\paren{\log\paren{1/\gamma}}^{1/8},\nonumber
\end{align}
with $\tilde{\mc{C}}_2:=\mc{C}_2+\frac{2\sqrt{2}}{\sqrt{1-\rho^2}}$, $\mc{C}_3=2\paren{\E\normsq{x^{(0)}}}^\frac12$, and $\mc{C}_4:=\frac{2D_\gamma}{(1-\rho)\sqrt{N}}+\frac{2\sigma}{\sqrt{1-\rho^2}}$.
This completes the proof of 
Theorem~\ref{thm:mt}.

\section{Proof of Corollary~\ref{cor:revised}}
\label{cor-proof}
\begin{proof}
    From Theorem~\ref{thm:mt}, for the average iterate $\bar{x}^{(k)}$, we have
    \begin{align}
        \mathcal W_2\left(\Law\left(\bar x^{(k)}\right),\pi\right)&\le(1-\mu\eta)^k\mathcal C_0+\left(\frac{\rho^{2k}-\left(A_\gamma(\eta)\right)^k}{\rho^2-A_\gamma(\eta)}\right)^{1/2}\frac{\rho}{\sqrt N}\mathcal C_1\nonumber
        \\
        &\qquad\qquad\qquad+\sqrt{\eta}\mathcal C_2+\mathcal C\gamma^{1/8}\big(\log(1/\gamma)\big)^{1/8},\label{eq:cor1}  
    \end{align}
    and
    \begin{align}
    \frac{1}{N}\sum_{i=1}^N\Wc_2\paren{\Law\paren{x_i^{(k)}},\pi}&\le (1-\mu\eta)^k\mc{C}_0+\left(\frac{\rho^{2k}-\left(A_\gamma(\eta)\right)^k}{\rho^2-A_\gamma(\eta)}\right)^{1/2}\frac{\rho}{\sqrt N}\mathcal C_1+\sqrt{\eta}\,\tilde{\mc{C}}_2+\frac{\rho^k}{\sqrt{N}}\mc{C}_3
    \nonumber
    \\
    &\qquad\qquad+\eta\,\mc{C}_4+\mc{C}\gamma^{1/8}\paren{\log\paren{1/\gamma}}^{1/8},
    \end{align}
    where $A_\gamma(\eta):=1-\frac{\mu\eta}{N}\left(1-\frac{\eta L_\gamma}{2N}\right)$. Since $\eta\le \frac{1}{L_\gamma+\mu}\le \frac{1}{L_\gamma}$, we have $\eta L_\gamma \le 1$, so that
    \[
    A_\gamma(\eta)=1-\frac{\mu\eta}{N}\paren{1-\frac{\eta L_\gamma}{2N}}\ge 1-\mu\eta.
    \]
    Now, choose $\eta \le \frac{1-\rho^2}{\mu}$. Then we have $1-\mu\eta\ge \rho^2$. This implies $A_\gamma(\eta)\ge \rho^2$, and thus we can rewrite the second term in~\eqref{eq:cor1} as
    \begin{align*}
        \left(\frac{\left(A_\gamma(\eta)\right)^k-\rho^{2k}}{A_\gamma(\eta)-\rho^2}\right)^{1/2}\le\frac{\left(A_\gamma(\eta)\right)^{k/2}}{\sqrt{A_\gamma(\eta)-\rho^2}}.    
    \end{align*}
    Therefore,
    \begin{align*}
        \mathcal W_2\big(\Law(\bar x^{(k)}),\pi\big)\le(1-\mu\eta)^k\mathcal C_0+\frac{\rho\,\left(A_\gamma(\eta)\right)^{k/2}}{\sqrt{N(A_\gamma(\eta)-\rho^2)}}\mathcal C_1+\sqrt{\eta}\,\mathcal C_2+\mathcal C\,\gamma^{1/8}\big(\log(1/\gamma)\big)^{1/8},
    \end{align*}
    and
    \begin{align}
    \frac{1}{N}\sum_{i=1}^N\Wc_2\paren{\Law\paren{x_i^{(k)}},\pi}&\le (1-\mu\eta)^k\mc{C}_0+\frac{\rho\,\left(A_\gamma(\eta)\right)^{k/2}}{\sqrt{N(A_\gamma(\eta)-\rho^2)}}\mathcal C_1+\sqrt{\eta}\,\tilde{\mc{C}}_2+\frac{\rho^k}{\sqrt{N}}\mc{C}_3
    \nonumber
    \\
    &\qquad\qquad+\eta\,\mc{C}_4+\mc{C}\gamma^{1/8}\paren{\log\paren{1/\gamma}}^{1/8}.
    \end{align}
    
Next, let us recall that 
$\eta<\min\left\{\frac{2N}{L_{\gamma}},\frac{1+\lambda_{N}^{W}}{L_{\gamma}},\frac{1}{L_{\gamma}+\mu}\right\}$
and $L_{\gamma}=L+\frac{2}{N\gamma}=\mc{O}(\gamma^{-1})$
such that 
\begin{align}\label{eta:constraint}
\eta=\mc{O}(\gamma^{-1}).
\end{align}

Then, let us recall that
 \begin{align}\label{mc:C}
        \mc{C}:=\paren{(A_*(\gamma))^\frac12+\paren{\frac{A_*(\gamma)}{2}}^\frac14}\sup_{0<\gamma\le \gamma_0}\hat{C},
    \end{align}
where $\hat{C}$ is defined in \eqref{hat:C:defn}
and
\begin{align}\label{A:gamma}
 A_*=\frac{d\,V_R\,\sqrt{2\tilde{\alpha}}\,e^{-m_{\ast}}}{rZ_\mk}\paren{1+\frac{\sqrt{2\tilde{\alpha}\gamma_{0}\log(1/\gamma_{0})}}{r}}^{d-1}+\frac{J_*}{Z_\mk},
\end{align}
where $V_{R}:=\frac{\pi^{d/2}}{\Gamma(\frac{d}{2}+1)}R^{d}$.
We also recall that
\begin{align*}
\mc{C}_0:=\paren{\paren{\E\normsq{\tilde{x}^{(0)}-\hat x^*_\gamma}}^{1/2}+\sqrt{2d\mu^{-1}}},
\qquad
\mc{C}_1:=2\sqrt{2}L_\gamma\paren{\E\normsq{x^{(0)}}}^\frac12,
\end{align*}
and
\begin{align*}
\mc{C}_2:&=\frac{1.65L_\gamma}{\mu}\sqrt{ d}+\frac{\sigma}{\sqrt{N\mu\paren{1-\frac{\eta L_\gamma}{2}}}}\\
    &\qquad+\paren{\frac{\eta}{\mu\paren{1-\frac{\eta L_\gamma}{2}}}+\frac{(1+\eta L_\gamma)^2}{\mu^2\paren{1-\frac{\eta L_\gamma}{2}}^2}}^\frac12\cdot\paren{\frac{8L_\gamma^2\eta D_\gamma^2}{N(1-\rho)^2}+\frac{8L_\gamma^2\eta\sigma^2}{(1-\rho^2)}+\frac{16L_\gamma^2  d }{(1-\rho^2)}}^\frac12,
\end{align*}
where $\hat x^*_\gamma$ is the unique minimizer of $\hat{u}^\gamma$, $\tilde{x}^{(0)}=\frac1N\sum_{i=0}^{N}x_i^{(0)}$, $D_\gamma$ is defined in~\eqref{eq:D}.

Let us spell out the dependence of $\mathcal{C}$, $\mathcal{C}_{0},\mathcal{C}_{1},\mathcal{C}_{2}$ on $\gamma$ and $d$. 

First, by using Stirling's approximation, we can derive from \eqref{A:gamma} that
\begin{equation}
A_{\ast}
=\mathcal{O}\left(\frac{\pi^{d/2}R^{d}}{e^{\frac{d}{2}\log(d/2)-\frac{d}{2}}}de^{\frac{\sqrt{2\tilde{\alpha}}}{r}d\sqrt{\gamma_{0}\log(1/\gamma_{0})}}\right)=\mc{O}(1),
\end{equation}
as $d\rightarrow\infty$.
By plugging this into \eqref{mc:C}, we get
\begin{equation}
\mc{C}=\mathcal{O}(1).
\end{equation}

Next, since $\hat x_{\gamma}^{\ast}$ is the unique minimizer of $\hat{u}^{\gamma}$ and as $\gamma\rightarrow 0$, $\hat x_{\gamma}^{\ast}$ will converge to the closure of the set $\mathcal{K}$, and this implies that
\begin{align}
\mathcal{C}_{0}=\mathcal{O}\left(\sqrt{d}\right), 
\qquad\text{as $\gamma\rightarrow 0$.}
\end{align}

After that, since $L_{\gamma}=L+\frac{2}{N\gamma}$, it follows that
\begin{align}
\mathcal{C}_{1}=\mathcal{O}(1/\gamma), 
\qquad\text{as $\gamma\rightarrow 0$.}
\end{align}

Now, in order to find how $\mc{C}_2$ behaves as $\gamma\to 0$, we first analyze $D_\gamma$ where 
    \begin{equation}
        \begin{split}
            D_\gamma^2&:=4L_\gamma^2\paren{1-\mu\eta\paren{1+\lambda_N^W-\eta L_\gamma}}^k\normsq{x^{(0)}-x^*_{\eta,\gamma}} + 8L_\gamma^2\frac{C_1^2\eta^2 N}{(1-\rho)^2}\\
         &\qquad\qquad+4\normsq{\nabla U^{\gamma}(x^*_\gamma)}+\frac{2L_\gamma^2(\eta \sigma^2 N+2dN)}{\mu\paren{1+\lambda_N^W-\eta L_\gamma}},
        \end{split}
        \label{eq:D1}
    \end{equation}
where we recall from \eqref{eq:c1} that
\begin{equation}\label{eq:c1:proof}
        C_1:=\paren{1+\frac{2(L_\gamma+\mu)}{\mu}}\sqrt{2L_\gamma\sum_{i=1}^N\paren{u^\gamma_i(0)-u^{\gamma,*}_i}},\quad\textrm{with}\quad u^{\gamma,*}_i:=\min_{x\in\R^d}u^{\gamma}_i(x).
        \end{equation}
    
    Here, for the first term in~\eqref{eq:D1}, we have 
    $$
    4L_\gamma^2\paren{1-\mu\eta\paren{1+\lambda_N^W-\eta L_\gamma}}^k\normsq{x^{(0)}-x^*_{\eta,\gamma}}=\mc{O}(L_\gamma^2)=\mc{O}(\gamma^{-2}).
    $$
    For the second term in~\eqref{eq:D1}, we have 
    $$
    8L_\gamma^2\frac{C_1^2\eta^2 N}{(1-\rho)^2}=\mc{O}(L_\gamma^2L_{\gamma}^{3}\eta^2)=\mc{O}(\gamma^{2\alpha-5}),
    $$ 
    by taking $\eta=\mc{O}(\gamma^\alpha)$ for some $\alpha\geq 1$ to be chosen later which satisfies the constraint \eqref{eta:constraint}, where we used the fact that $C_{1}=\mathcal{O}(L_{\gamma}^{3/2})=\mathcal{O}(1/\gamma^{3/2})$ from \eqref{eq:c1:proof} since 
    $$
    0\leq u_{i}^{\gamma}(0)-u_{i}^{\gamma,\ast}\leq u_{i}^{\gamma}(0)=f_{i}(0)=\mc{O}(1),
    $$ 
    as $\gamma\rightarrow 0$, where the first equality above is due to $0\in\mathcal{K}$.

    For the last term in \eqref{eq:D1}, we have 
    $$
    \frac{2L_\gamma^2(\eta \sigma^2 N+2dN)}{\mu\paren{1+\lambda_N^W-\eta L_\gamma}}=\mc{O}(dL_\gamma^2)=\mc{O}(d\gamma^{-2}). 
    $$
    Therefore,
   \[
    D_\gamma^2=\mc{O}\paren{d\gamma^{-2}+d\gamma^{2\alpha-5}},\qquad D_\gamma=\begin{cases}
        \mc{O}\paren{\sqrt{d}\gamma^{\alpha-5/2}}, & 1\leq\alpha < 3/2,\\
        \mc{O}\paren{\sqrt{d}\gamma^{-1}},&\alpha\ge 3/2.
    \end{cases}
    \]
    For $\alpha\geq 1$, we have $\eta L_\gamma = \mc{O}\paren{\gamma^{\alpha-1}}$. Therefore, 
    \[
    L_\gamma^2\eta\,D_\gamma^2+L_\gamma^2\eta+dL_\gamma^2=\mc{O}\paren{\gamma^{-2}\gamma^\alpha D_\gamma^2+\gamma^{-2}\gamma^\alpha+\gamma^{-2}}=\mc{O}\paren{\gamma^{\alpha-2}D_\gamma^2+\gamma^{\alpha-2}+d\gamma^{-2}}.
    \]
    Now for $1\le \alpha < 3/2$, we have $D_\gamma^2=\mc{O}(d\gamma^{2\alpha-5})$. Then,
    \[
    \mc{C}_2=\mc{O}\left(\sqrt{d}\gamma^{-1}\right)+\mc{O}(1)+\mc{O}\paren{\sqrt{d}\gamma^{\frac{3\alpha-7}{2}}}=\mc{O}\paren{\sqrt{d}\gamma^{\frac{3\alpha-7}{2}}}.
    \]
    If $3/2\le \alpha <2$, then
    \[
    \mc{C}_2=\mc{O}\left(\sqrt{d}\gamma^{-1}\right)+\mc{O}(1)+\mc{O}\paren{\sqrt{d}\gamma^{\frac{\alpha-4}{2}}}=\mc{O}\paren{\sqrt{d}\gamma^{\frac{\alpha-4}{2}}}.
    \]
    If $\alpha\ge 2$, then we have $\mc{C}_2=\mc{O}\paren{\sqrt{d}\gamma^{-1}}$. Thus we can summarize the behavior of $\mc{C}_2$ as $\gamma\to 0$ depending on some $\alpha\ge 1$ as follows
    \[
    \mc{C}_2=\begin{cases}
        \mc{O}\paren{\sqrt{d}\gamma^{\frac{3\alpha-7}{2}}},& 1\le \alpha <3/2,\\
        \mc{O}\paren{\sqrt{d}\gamma^{\frac{\alpha-4}{2}}}, & 3/2\le\alpha< 2,\\
        \mc{O}\left(\sqrt{d}\gamma^{-1}\right), & \alpha\ge 2. 
    \end{cases}
    \]
    By taking $\eta=\mc{O}(\gamma^{\alpha}/d)$, we get
\begin{align}\label{eta:C:2}
   \sqrt{\eta} \mc{C}_2=\begin{cases}
        \mc{O}\paren{\gamma^{\frac{4\alpha-7}{2}}},& 1\le \alpha <3/2,\\
        \mc{O}\paren{\gamma^{\alpha-2}}, & 3/2\le\alpha< 2,\\
        \mc{O}\left(\gamma^{\frac{\alpha}{2}-1}\right), & \alpha\ge 2. 
    \end{cases}
    \end{align}
     Now for a given $\varepsilon>0$, we choose $\gamma$ and $\eta$ such that 
    \[
    \sqrt{\eta}\,\mc{C}_2\le \frac{\varepsilon}{6},
    \qquad\text{and}\qquad \mc{C}\gamma^{1/8}(\log(1/\gamma))^{1/8}\le \frac{\varepsilon}{6}. 
    \]
    Thus, $\varepsilon\rightarrow 0$, we require $\gamma\rightarrow 0$ and $\sqrt{\eta}\mathcal{C}_{2}\rightarrow 0$. From \eqref{eta:C:2}, $\sqrt{\eta}\mathcal{C}_{2}\rightarrow 0$ as $\gamma\rightarrow 0$ if and only if $\alpha>2$.
     For $\alpha>2$, with $\eta=\mc{O}(\gamma^{\alpha}/d)$, we have $\sqrt{\eta}\,\mc{C}_2=\mc{O}\paren{\gamma^{\frac{\alpha}{2}-1}}$, and 
     and we need to choose $\gamma$ and $\eta$ such that $\eta=\mc{O}(\gamma^{\alpha}/d)$ with $\alpha>2$ and 
      \[
    \mc{O}\left(\sqrt{d}\gamma^{\frac{\alpha}{2}-1}\right)\le \frac{\varepsilon}{6},
    \qquad\text{and}\qquad \mc{C}\gamma^{1/8}(\log(1/\gamma))^{1/8}\le \frac{\varepsilon}{6}. 
    \]
     Then matching  $\gamma^{\frac{\alpha}{2}-1}$ with $\gamma^{1/8}$ we obtain $\alpha=\frac94$.
    From these, we choose $\gamma=\mc{O}\left(\varepsilon^8\right)$ to obtain
    \begin{align}
    \mc{C}\gamma^{1/8}(\log(1/\gamma))^{1/8}
    \le \tilde{\mc{O}}(\varepsilon),
    \end{align}
     where $\tilde{\mc{O}}$ hides the logarithmic dependence on $\varepsilon$.
    and we further choose $\eta=\mc{O}(\gamma^{9/4}/d)=\mc{O}(\varepsilon^{18}/d)$ such that
    \[
    \sqrt{\eta}\mc{C}_2+\mc{C}\gamma^{1/8}(\log(1/\gamma))^{1/8}\le \tilde{\mc{O}}(\varepsilon),
    \]
    where $\tilde{\mc{O}}$ hides the logarithmic dependence on $\varepsilon$. Also, recall that,
    \[
    \tilde{\mc{C}}_2:=\mc{C}_2+\frac{2\sqrt{2}}{\sqrt{1-\rho^2}},\qquad \mc{C}_3:=2\paren{\E\normsq{x^{(0)}}}^{\frac12},\qquad \mc{C}_4:=\frac{2D_\gamma}{(1-\rho)\sqrt{N}}+\frac{2\sigma}{\sqrt{1-\rho^2}},
    \]
    where $\tilde{\mc{C}}_2=\mc{O}(\sqrt{d}\gamma^{-1}), \mc{C}_3=\mc{O}(1)$, and $\mc{C}_4=\mc{O}(\sqrt{d}\gamma^{-1})$ as $\gamma\to 0$. Accordingly, we get
    \[
    \sqrt{\eta}\,\tilde{\mc{C}}_2+\eta\mc{C}_{4}
    +\mc{C}\gamma^{1/8}(\log(1/\gamma))^{1/8}\le
    \tilde{\mathcal{O}}(\varepsilon),
    \]
    where $\tilde{\mathcal{O}}$
    hides the logarithmic dependence on $\varepsilon$.

    Now, given such $\eta,\gamma > 0$, it remains to choose $K$ so that
    \begin{equation}
        \paren{1-\mu\eta}^K\mc{C}_0\le \frac{\varepsilon}{3}, \qquad \frac{\rho\,\left(A_\gamma(\eta)\right)^{K/2}}{\sqrt{N\paren{A_\gamma(\eta)-\rho^2}}}\,\mc{C}_1\le \frac{\varepsilon}{3},\label{eq:corr2}
    \end{equation}
    which implies that
    \begin{equation}
    \mathcal W_2\left(\Law\left(\bar x^{(k)}\right),\pi\right)
    \leq\tilde{\mathcal{O}}(\varepsilon),
    \end{equation}
    and
    \begin{equation}
        \frac{1}{N}\sum_{i=1}^N\Wc_2\paren{\Law\paren{x_i^{(k)}},\pi}\leq \tilde{\mathcal{O}}(\varepsilon),
    \end{equation}
    where $\tilde{\mathcal{O}}$ hides the logarithmic dependence on $\varepsilon$. 
    From the first inequality in~\eqref{eq:corr2}, by using $\mathcal{C}_{0}=\mathcal{O}(\sqrt{d})$, we have
    \[
    K\ge \frac{\log(3\mc{C}_0/\varepsilon)}{|\log(1-\mu\eta)|}\asymp \frac{\log(3\mc{C}_0/\varepsilon)}{\mu\eta}=\mc{O}\paren{\frac{1}{\eta}\log\paren{\frac{\sqrt{d}}{\varepsilon}}}.
    \]
    From the second inequality in~\eqref{eq:corr2}, we have
    \begin{align*}
        \frac{\rho\,\left(A_\gamma(\eta)\right)^{K/2}}{\sqrt{N\paren{A_\gamma(\eta)-\rho^2}}}\,\mc{C}_1\le \frac{\varepsilon}{3},
    \end{align*}
    which implies that
    \begin{align*}
        \left(A_\gamma(\eta)\right)^{K/2} \le \frac{\varepsilon\sqrt{N\paren{A_\gamma(\eta)-\rho^2}}}{3\rho\,\mc{C}_1}.
    \end{align*}
    Therefore, by using $\mathcal{C}_{1}=\mc{O}(1/\gamma)=\mc{O}(1/\varepsilon^{8})$, we get
    \begin{align*}
        K\ge \frac{2\log\paren{\frac{3\rho\,\mc{C}_1}{\varepsilon\sqrt{N\paren{A_\gamma(\eta)-\rho^2}}}}}{|\log(A_\gamma(\eta))|}\asymp \frac{2}{\mu\eta}\log\paren{\frac{3\rho\,\mc{C}_1}{\varepsilon\sqrt{N\paren{A_\gamma(\eta)-\rho^2}}}}
        =\mc{O}\paren{\frac{1}{\eta}\log\paren{\frac{1}{\varepsilon^{9}}}}.
    \end{align*}
    Finally, since $\eta=\mc{O}\paren{\varepsilon^{18}/d}$, we conclude that $$\mathcal W_2\left(\Law\left(\bar x^{(k)}\right),\pi\right)
    \leq\tilde{\mathcal{O}}(\varepsilon),$$ 
    and 
    $$\frac{1}{N}\sum_{i=1}^N\Wc_2\paren{\Law\paren{x_i^{(k)}},\pi}\leq \tilde{\mathcal{O}}(\varepsilon),$$ 
    provided that  
    \begin{align*}
        K=\mc{O}\paren{\frac{d}{\varepsilon^{18}}\log\paren{\frac{\sqrt{d}}{\varepsilon^{9}}}},
    \end{align*}
    which gives the iteration complexity.
This completes the proof.
\end{proof}
\section{Proofs of the Technical Lemmas}
\subsection{Proof of Lemma~\ref{lem:newlemma}}
\begin{proof}
    \label{lem:proof-newlemma}
    We first compute an upper bound on the Kullback-Leibler (KL) divergence between $\pi^{\gamma}$ and $\pi$, and then provide an upper bound in terms of the 2-Wasserstein distance using the weighted Csisz\'{a}r--Kullback--Pinsker
    inequality by \cite{BV} (see e.g. Lemma~B.1 in \cite{gurbuzbalaban2024penalized}).

    First, we recall that if $P$ and $Q$ are probability distributions of  continuous random variables with densities $p$ and $q$ on $\mathbb{R}^{d}$ such that $P$ is absolutely continuous with respect to $Q$, the Kullback and Leibler (KL) divergence~\citep{kullback1951information} between $P$ and $Q$ is defined as
    \[
    D(P \parallel Q) = \int_{\mathbb{R}^{d}} p(x) \log \left( \frac{p(x)}{q(x)} \right) dx.
    \]

    By definition $q_\mk^\gamma(x)=0$ for any $x\in\mk$, and $q_\mk^\gamma(x)>0$ for any $x\notin \mk$. Assume that $e^{-f}$ is integrable over $\mk$. Define the normalizing constants
    \[
    Z_\mk:=\int_{\mathcal{K}}e^{-f(y)}dy,\qquad Z_{\gamma}:=\int_{\R^d}e^{-f(y)-\frac{1}{2\gamma}\|y-\proj(y)\|^2}dy.
    \]
    Then, the corresponding densities are
    \[
    \pi(x)=\frac{e^{-f(x)}\mb{1}_{x\in\mathcal K}}{\int_{\mathcal{K}}e^{-f(y)}dy};\qquad \pi^{\gamma}(x)=\frac{e^{-f(y)-\frac{1}{2\gamma}\|y-\proj(y)\|^2}}{\int_{\R^d}e^{-f(y)-\frac{1}{2\gamma}\|y-\proj(y)\|^2}dy}.
    \]
    Thus, the KL divergence between $\pi$ and $\pi^{\gamma}$ can be computed as
    \begin{align*}
        D(\pi\Vert \pi^{\gamma})&=\int_{\R^d}\log\paren{\frac{\pi(x)}{\pi^{\gamma}(x)}}\pi(x)dx\\
        &=\int_{\mathcal{K}}\log \paren{\frac{e^{-f(x)}\mb{1}_{x\in\mathcal K}}{\int_{\mathcal{K}}e^{-f(y)}dy}\cdot\frac{\int_{\R^d}e^{-f(y)-\frac{1}{2\gamma}\|y-\proj(y)\|^2}dy}{e^{-f(y)-\frac{1}{2\gamma}\|y-\proj(y)\|^2}}}\frac{e^{-f(x)}\mb{1}_{x\in\mathcal K}}{\int_{\mathcal{K}}e^{-f(y)}dy}dx.
    \end{align*}
    Note that on $\mathcal{K}$, we have $\proj(x)=x$. Therefore,
    \begin{align*}
        D(\pi\Vert \pi^{\gamma})&=\int_{\mathcal K}\log\paren{\frac{\int_{\R^d}e^{-f(y)-\frac{1}{2\gamma}\|y-\proj(y)\|^2}dy}{\int_{\mathcal{K}}e^{-f(y)}dy}}\frac{e^{-f(x)}\mb{1}_{x\in\mathcal K}}{\int_{\mathcal{K}}e^{-f(y)}dy}dx\\
        &=\log\paren{\frac{\int_{\R^d}e^{-f(y)-\frac{1}{2\gamma}\|y-\proj(y)\|^2}dy}{\int_{\mathcal{K}}e^{-f(y)}dy}}\int_{\mathcal{K}}\frac{e^{-f(x)}\mb{1}_{x\in\mathcal K}}{\int_{\mathcal{K}}e^{-f(y)}dy}dx\\
        &=\log\paren{\frac{\int_{\R^d}e^{-f(y)-\frac{1}{2\gamma}\|y-\proj(y)\|^2}dy}{\int_{\mathcal{K}}e^{-f(y)}dy}}\\
        &=\log\paren{\frac{\int_{\mathcal{K}}e^{-f(y)}dy+\int_{\R^d\setminus \mathcal{K}}e^{-f(y)-\frac{1}{2\gamma}\|y-\proj(y)\|^2}dy}{\int_{\mathcal{K}}e^{-f(y)}dy}}.
        \end{align*}
        Define $T(x)=g(\gamma_\mk(x))=\left(\gamma_\mk(x)\right)^2,$ where $\gamma_\mk(x)=\|x-\proj(x)\|$ is the distance of the point $x$ to the set $\mk$. Therefore
        \begin{equation}
        D(\pi\Vert \pi^{\gamma})\le \frac{\int_{\R^d\setminus \mathcal{K}}e^{-f(y)-\frac{1}{2\gamma}T(y)}dy}{\int_{\mathcal{K}}e^{-f(y)}dy}.
        \label{eq:kldiv}
    \end{equation}
    Notice that for any $\omega>0$,
    \begin{align}
        &\int_{\R^d\setminus \mathcal{K}}e^{-f(y)-\frac{1}{2\gamma}T(y)}dy\nonumber\\
        &=\int_{y\in\R^d\setminus \mathcal{K}:T(y)\le \omega}e^{-f(y)-\frac{1}{2\gamma}T(y)}dy+\int_{y\in\R^d\setminus \mathcal{K}:T(y)> \omega}e^{-f(y)-\frac{1}{2\gamma}T(y)}dy\nonumber\\
        &\le \left|y\in\R^d\setminus\mk:T(y)\le \omega\right|e^{-\inf_{y\in\R^d\setminus\mk:T(y)\le \omega}f(y)}+e^{-\frac{\omega}{2\gamma}}\int_{\R^d\setminus\mk}e^{-f(y)-\frac{1}{2\gamma}T(y)}dy.\label{eq:lebesgue}
    \end{align}
    We now state a result from~\citet{gurbuzbalaban2024penalized} (Lemma 2.4, without a proof) to provide an upper bound on the Lebesgue measure of the set $\mk$. 
    
    \begin{lemma}
        \label{lem:lebsgue}
        Assume the constraint set $\mk$ is a bounded set containing an open ball with radius $r>0$. Let $T(x)=g(\gamma_\mk(x))$ and $g:\R_{\ge0}\rightarrow \R_{\ge0}$ is a strictly increasing function with $g(0)=0$ and $g(x)\to \infty$ as $x\to \infty$. Then for $\tau>0$,
        \begin{equation*}
        \left|x\in\R^d\setminus\mk:T(x)\le \tau\right|\le \paren{\paren{1+\frac{g^{-1}(\tau)}{r}}^d-1}|\mk|,
        \end{equation*}
        where $|\cdot|$ denotes the Lebesgue measure and $g^{-1}$ is the inverse function of $g$.
    \end{lemma}
    
    Using the result from Lemma~\ref{lem:lebsgue} with $g(x)=x^{2}$ such that $g^{-1}(x)=\sqrt{x}$ and taking $\omega=2\tilde{\alpha}\gamma\log(1/\gamma)$ for any $\tilde{\alpha}\ge \frac12$, we obtain from~\eqref{eq:lebesgue}:
    \begin{align}
        \int_{\R^d\setminus \mathcal{K}}e^{-f(y)-\frac{1}{2\gamma}T(y)}dy
        &\le\left|y\in\R^d\setminus\mk:T(y)\le 2\tilde{\alpha}\gamma\log(1/\gamma)\right|e^{-\inf_{y\in\R^d\setminus\mk:T(y)\le 2\tilde{\alpha}\gamma\log(1/\gamma)}f(y)}\nonumber\\
        &\qquad\qquad+\gamma^{\tilde{\alpha}}\int_{\R^d\setminus\mk}e^{-f(y)-\frac{1}{2\gamma}T(y)}dy\nonumber\\
        &\le \paren{\paren{1+\frac{\sqrt{2\tilde{\alpha}\gamma\log(1/\gamma)}}{r}}^d-1}|\mk|e^{-\inf_{y\in\R^d\setminus\mk:T(y)\le 2\tilde{\alpha}\gamma\log(1/\gamma)}f(y)}\nonumber\\
        &\qquad\qquad+\gamma^{\tilde{\alpha}}\int_{\R^d\setminus\mk}e^{-f(y)-\frac{1}{2\gamma}T(y)}dy\nonumber\\
        &\le \paren{\paren{1+\frac{\sqrt{2\tilde{\alpha}\gamma\log(1/\gamma)}}{r}}^d-1}V_R\, e^{-\inf_{y\in\R^d\setminus\mk:T(y)\le 2\tilde{\alpha}\gamma\log(1/\gamma)}f(y)}\nonumber\\
        &\qquad\qquad+\gamma^{\tilde{\alpha}}\int_{\R^d\setminus\mk}e^{-f(y)-\frac{1}{2\gamma}T(y)}dy,\label{eq:lebesgue1}
    \end{align}
    where $V_R:=\frac{\pi^{d/2}}{\Gamma\paren{\frac{d}{2}+1}}R^d$. Since $\mk$ is contained in an Euclidean ball with radius $R$, we have $|\mk|\le V_R$. Thus, we obtain inequality \eqref{eq:lebesgue1}.
    Therefore, the upper bound in~\eqref{eq:kldiv} can be written as:
    \begin{equation}
        D(\pi\Vert\pi^\gamma) \le \paren{\paren{1+\frac{\sqrt{2\tilde{\alpha}\gamma\log(1/\gamma)}}{r}}^d-1}\frac{V_R\,e^{-m(\gamma)}}{Z_\mk}+\gamma^{\tilde{\alpha}}\frac{J(\gamma)}{Z_\mk},\label{eq:lebesgue2}
    \end{equation}
    where 
    \begin{equation*}
        m(\gamma):=\inf_{y\notin\mk:T(y)\le 2\tilde{\alpha}\gamma\log(1/\gamma)}f(y),\qquad\qquad J(\gamma):= \int_{\R^d\setminus\mk}e^{-f(y)-\frac{1}{2\gamma}T(y)}dy.
    \end{equation*}
    Take $\gamma_0\in(0,1/e)$, and set
    \[
    m_*:=\inf_{0<\gamma\le\gamma_0}m(\gamma),\qquad J_*:=\sup_{0<\gamma\le \gamma_0}J(\gamma).
    \]
    Then for every $0<\gamma\le \gamma_0$,
    \[
    m(\gamma)\ge m_*,\qquad J(\gamma)\le J_*.
    \]
 Now, the first term of~\eqref{eq:lebesgue2} can be bounded by using the inequality
    \[
    (1+u)^d-1 \le d u(1+u)^{d-1}, \qquad u\ge 0.
    \]
    Therefore,
    \begin{align*}  &\paren{\paren{1+\frac{\sqrt{2\tilde{\alpha}\gamma\log(1/\gamma)}}{r}}^d-1}\frac{V_R\,e^{-m(\gamma)}}{Z_\mk}\\
        &\le d\paren{\frac{\sqrt{2\tilde{\alpha}\gamma\log(1/\gamma)}}{r}}\paren{1+\frac{\sqrt{2\tilde{\alpha}\gamma\log(1/\gamma)}}{r}}^{d-1}\frac{V_R}{Z_\mk}e^{-m(\gamma)}\\
        &=\frac{d\,V_R\,\sqrt{2\tilde{\alpha}}\,e^{-m(\gamma)}}{rZ_\mk}\paren{1+\frac{\sqrt{2\tilde{\alpha}\gamma\log(1/\gamma)}}{r}}^{d-1}\sqrt{\gamma\log(1/\gamma)}\\
        &=A_1(\gamma)\sqrt{\gamma\log(1/\gamma)},
    \end{align*}
    where 
\[
A_1(\gamma):=\frac{d\,V_R\,\sqrt{2\tilde{\alpha}}\,e^{-m(\gamma)}}{rZ_\mk}\paren{1+\frac{\sqrt{2\tilde{\alpha}\gamma\log(1/\gamma)}}{r}}^{d-1},
    \]
    where $V_R:=\frac{\pi^{d/2}}{\Gamma\paren{\frac{d}{2}+1}}R^d$.
    
    Since we have $\tilde{\alpha}\ge \frac12$, for $0<\gamma\le \gamma_0$, we have $\gamma^{\tilde{\alpha}}\le\sqrt{\gamma}\le\sqrt{\gamma\log(1/\gamma)}$. Thus, the second part of~\eqref{eq:lebesgue2} can be written as:
    \begin{align*}      \gamma^{\tilde{\alpha}}\frac{J(\gamma)}{Z_\mk}&\le \gamma^{\tilde{\alpha}}\frac{J_*}{Z_\mk}\le \frac{J_*}{Z_\mk}\sqrt{\gamma\log(1/\gamma)}.
    \end{align*}
    Now, from~\eqref{eq:lebesgue2}, 
    \begin{equation}\label{eq:lebesgue3}
        D\paren{\pi\Vert\pi^\gamma}\le \paren{A_1(\gamma)+\frac{J_*}{Z_\mk}}\sqrt{\gamma\log(1/\gamma)}=A_*\sqrt{\gamma\log(1/\gamma)},
    \end{equation}
    where
    \begin{equation}\label{eq:lebesgue4}
        A_*=\frac{d\,V_R\,\sqrt{2\tilde{\alpha}}\,e^{-m_{\ast}}}{rZ_\mk}\paren{1+\frac{\sqrt{2\tilde{\alpha}\gamma_{0}\log(1/\gamma_{0})}}{r}}^{d-1}+\frac{J_*}{Z_\mk},
    \end{equation}
    where we used the fact that the map $x\mapsto x\log(1/x)$ is increasing on $x\in(0,1/e)$.
    Now we are ready to use the weighted Csisz\'{a}r--Kullback--Pinsker
    inequality by \cite{BV} (see e.g. Lemma~B.1 in \cite{gurbuzbalaban2024penalized}) to provide a bound in the 2-Wasserstein distance:
    \begin{equation}
    \Wc_2\paren{\pi,\pi^\gamma} \le \hat{C}\paren{\left(D\paren{\pi\Vert\pi^\gamma}\right)^\frac12+\paren{\frac{D\paren{\pi\Vert\pi^\gamma}}{2}}^\frac14},\label{eq:lebesgue4:2}
    \end{equation}
    where 
    \begin{equation}\label{hat:C:defn}
\hat{C}=2\inf_{\hat{x}\in\R^d,\tilde{\alpha}\ge\frac12}\paren{\frac{1}{\tilde{\alpha}}\paren{\frac32+\log\paren{\int_{\R^d}e^{-\tilde{\alpha}\normsq{x-\hat{x}}}d\pi^\gamma}}}^\frac12.
    \end{equation}
    By plugging the results from~\eqref{eq:lebesgue3} into~\eqref{eq:lebesgue4:2}, we have
    \begin{align*}
        \Wc_2\paren{\pi,\pi^\gamma} &\le \hat{C}\paren{(A_*)^\frac12\paren{\gamma\log\paren{\frac1\gamma}}^\frac14+\paren{\frac{A_*}{2}}^\frac14\paren{\gamma\log\paren{\frac1\gamma}}^\frac18}\\
        &\le \hat{C}\paren{(A_*)^\frac12\paren{\gamma\log\paren{\frac1\gamma}}^\frac18+\paren{\frac{A_*}{2}}^\frac14\paren{\gamma\log\paren{\frac1\gamma}}^\frac18}\\
        &=\hat{C}\paren{(A_*)^\frac12+\paren{\frac{A_*}{2}}^\frac14}\paren{\gamma\log\paren{\frac1\gamma}}^\frac18.
    \end{align*}
    Let 
    \begin{align}
        \mc{C}:=\paren{(A_*)^\frac12+\paren{\frac{A_*}{2}}^\frac14}\sup_{0<\gamma\le \gamma_0}\hat{C},\label{eq:lebesgue5}
    \end{align}
    where $A_*$ is defined in~\eqref{eq:lebesgue4} and $\hat{C}$ is defined in \eqref{hat:C:defn}.
    Thus,
    \[
    \Wc_2\paren{\pi,\pi^\gamma}\le \mc{C}\paren{\gamma\log\paren{\frac1\gamma}}^\frac18.
    \]
This completes the proof.
\end{proof}

\subsection{Proof of Lemma~\ref{lem:gradbound}}
\label{lemproof:gradbound}

\begin{proof}
    Let us define
    \begin{equation}
        U^{\gamma}_{\Wc,\eta}(x):=\frac{1}{2\eta}x^{\top}(\mb{I}_{Nd}-\Wc)x+U^{\gamma}(x)=\frac{1}{2\eta}x^{\top}(\mb{I}_{Nd}-\Wc)x+F(x)+\frac{1}{2N\gamma}G(x).
    \end{equation}
    Based on Lemma 2.6 in~\citet{gurbuzbalaban2024penalized}, $G(x)$ is $\ell-$smooth with $\ell=4$. Therefore, $U^{\gamma}_{\Wc,\eta}$ is $\mu$-strongly convex and $L_\eta$-smooth with $L_{\eta}=L+\frac{1-\lambda_N^W}{\eta}+\frac{2}{N\gamma}$. Then equation~\eqref{eq:concatenatedxk} can be written in terms of $U^{\gamma}_{\Wc,\eta}$ as follows:
    \begin{equation}
        x^{(k+1)}=x^{(k)} - \eta \nabla U^{\gamma}_{\Wc,\eta}\paren{x^{(k)}} - \eta \xi^{(k+1)}+\sqrt{2\eta}w^{(k+1)}.
        \label{eq:gradbound2}
    \end{equation}
    Recall that we defined $x^*_\gamma=\sqb{(x_*^{\gamma})^{\top},(x_*^\gamma)^{\top},\cdots, (x_*^\gamma)^{\top}}^{\top}$ as a vector in the $Nd$ dimension, with $x_*^\gamma\in \R^d$. Let us also define $x^*_{\eta,\gamma}$ as the minimizer of $U^{\gamma}_{\Wc,\eta}$. Since the gradients $\nabla U^{\gamma}(x)$ are $L_{\gamma}$-Lipschitz with $L_{\gamma}=L+\frac{2}{N\gamma}$, we have
    \begin{align}
        \E\normsq{\nabla U^\gamma\paren{x^{(k)}}}&=\E\normsq{\nabla U^\gamma\paren{x^{(k)}}-\nabla U^\gamma\paren{x^*_{\eta,\gamma}}+\nabla U^\gamma\paren{x^*_{\eta,\gamma}}}\nonumber\\
        &\le 2\E \normsq{\nabla U^\gamma\paren{x^{(k)}}-\nabla U^\gamma\paren{x^*_{\eta,\gamma}}}+2\normsq{\nabla U^\gamma\paren{x^*_{\eta,\gamma}}}\nonumber\\
        &\le 2 L_\gamma^2\E\normsq{x^{(k)}-x^*_{\eta,\gamma}}+2\normsq{\nabla U^\gamma\paren{x^*_{\eta,\gamma}}}\nonumber\\
        &\le 2 L_\gamma^2\E\normsq{x^{(k)}-x^*_{\eta,\gamma}}+4\normsq{\nabla U^\gamma\paren{x^*_{\eta,\gamma}}-\nabla U^\gamma\paren{x^*_\gamma}}+4\normsq{\nabla U^\gamma\paren{x^*_\gamma}}\nonumber\\
        &\le 2 L_\gamma^2\E\normsq{x^{(k)}-x^*_{\eta,\gamma}}+4L_\gamma^2\normsq{x^*_{\eta,\gamma}-x^*_\gamma}+4\normsq{\nabla U^\gamma\paren{x^*_\gamma}}.\label{eq:gradbound1}
    \end{align}
    To control the term $\norm{x_{\eta,\gamma}^*-x^*_\gamma}$ in~\eqref{eq:gradbound1}, we state the following result from~\citet{yuan2016convergence}:
    \begin{lemma}
        \label{lem:xetastar-xstar}
        For any $\eta \le \min\left\{\frac{1+\lambda_N^W}{L_\gamma}, \frac{1}{L_\gamma+\mu}\right\}$, we have
        \begin{equation*}
            \norm{x_{\eta,\gamma}^*-x^*_\gamma}\le C_1\frac{\eta\sqrt{N}}{1-\rho},\quad\textrm{with}\quad \rho:=\max\left\{\left|\lambda_2^W\right|,\left|\lambda_N^W\right|\right\},
        \end{equation*}
        where 
        \begin{equation}
        C_1:=\paren{1+\frac{2(L_\gamma+\mu)}{\mu}}\sqrt{2L_\gamma\sum_{i=1}^N\paren{u^\gamma_i(0)-u^{\gamma,*}_i}},\quad\textrm{with}\quad u^{\gamma,*}_i:=\min_{x\in\R^d}u^{\gamma}_i(x),\label{eq:c1} 
        \end{equation}
        where $u_{i}^{\gamma}(x):=f_{i}(x)+q_{\mathcal{K}}^{\gamma}(x)$.
    \end{lemma}
     The proof of Lemma~\ref{lem:xetastar-xstar} can be found in Theorem 1 in~\citet{yuan2016convergence} with some adjustments in notations.

     Then, we control the term $\E\normsq{x^{(k)}-x^*_{\eta,\gamma}}$ in~\eqref{eq:gradbound1}. From~\eqref{eq:gradbound2}, we have
     \begin{align}
         x^{(k+1)}-x^*_{\eta,\gamma}&=x^{(k)}-x^*_{\eta,\gamma}-\eta \nabla U^{\gamma}_{\Wc,\eta}\paren{x^{(k)}}-\eta \xi^{(k+1)}+\sqrt{2\eta}w^{(k+1)},\nonumber
     \end{align}
    which implies that
    \begin{align}
        \E\normsq{x^{(k+1)}-x^*_{\eta,\gamma}}&=\normsq{x^{(k)}-x^*_{\eta,\gamma}-\eta \nabla U^{\gamma}_{\Wc,\eta}\paren{x^{(k)}}-\eta \xi^{(k+1)}+\sqrt{2\eta}w^{(k+1)}}\nonumber\\
        &=\E\normsq{x^{(k)}-x_{\eta,\gamma}^*}-2\eta \E \lrangle{x^{(k)}-x^*_{\eta,\gamma},\nabla U^{\gamma}_{\Wc,\eta}\paren{x^{(k)}}}\nonumber\\
        &\qquad\qquad+\eta^2\E\normsq{\nabla U^{\gamma}_{\Wc,\eta}\paren{x^{(k)}}}+\eta^2\E\normsq{\xi^{(k+1)}}+2\eta\,dN.\label{eq:gradbound3}
    \end{align}
     Since $U^{\gamma}_{\Wc,\eta}$ is $\mu$-strongly convex and $L_{\eta}$-smooth, therefore
     \begin{align}
         L_\eta\lrangle{\nabla \ugc(z)-\nabla \ugc(y), z-y }&\ge \normsq{\nabla\ugc(z)-\nabla\ugc(y)}\qquad \forall z,y\in \R^d,\\
         \lrangle{z-y,\nabla \ugc(z)-\nabla\ugc(y)}&\ge \mu\normsq{z-y},\qquad \forall z,y\in \R^d.
     \end{align}
     Also, from the first-optimality condition, we have $\nabla\ugc\paren{x^*_{\eta,\gamma}}=0$. Thus, from~\eqref{eq:gradbound3}, we obtain
     \begin{align}
         &\E\normsq{x^{(k+1)}-x^*_{\eta,\gamma}}
         \nonumber
         \\
         &= \E \normsq{x^{(k)}-x^*_{\eta,\gamma}}-2\eta \E\lrangle{x^{(k)}-x^*_{\eta,\gamma},\nabla\ugc\paren{x^{(k)}}-\nabla\ugc\paren{x^*_{\eta,\gamma}}}\nonumber\\
         &\qquad +\eta^2\E\normsq{\nabla U^{\gamma}_{\Wc,\eta}\paren{x^{(k)}}-\nabla\ugc\paren{x^*_{\eta,\gamma}}}+\eta^2\E\normsq{\xi^{(k+1)}}+2\eta\,dN\nonumber\\
         &\le \E \normsq{x^{(k)}-x^*_{\eta,\gamma}}-2\eta \E\lrangle{x^{(k)}-x^*_{\eta,\gamma},\nabla\ugc\paren{x^{(k)}}-\nabla\ugc\paren{x^*_{\eta,\gamma}}}\nonumber\\
         &\qquad+\eta^2L_\eta\E\lrangle{x^{(k)}-x^*_{\eta,\gamma},\nabla\ugc\paren{x^{(k)}}-\nabla\ugc\paren{x^*_{\eta,\gamma}}}+\eta^2\sigma^2N+2\eta\,dN\nonumber\\
         &=\E \normsq{x^{(k)}-x^*_{\eta,\gamma}}-2\eta\paren{1-\eta L_\eta/2} \E\lrangle{x^{(k)}-x^*_{\eta,\gamma},\nabla\ugc(x^{(k)})-\nabla\ugc\paren{x^*_{\eta,\gamma}}}\nonumber\\
         &\qquad+\eta^2\sigma^2N+2\eta\,dN\nonumber\\
         &\le \E \normsq{x^{(k)}-x^*_{\eta,\gamma}} - 2\eta\mu\paren{1-\frac{\eta L_\eta}{2}}\E\normsq{x^{(k)}-x^*_{\eta,\gamma}}+\eta^2\sigma^2N+2\eta\,dN\nonumber\\
         &=\paren{1-2\mu\eta\paren{1-\frac{\eta L_\eta}{2}}}\E\normsq{x^{(k)}-x^*_{\eta,\gamma}}+\eta^2\sigma^2N+2\eta\,dN\nonumber\\
         &=\paren{1-\mu\eta\paren{1+\lambda_N^W-\eta L_\gamma}}\E\normsq{x^{(k)}-x^*_{\eta,\gamma}}+\eta^2\sigma^2N+2\eta\,dN,\nonumber
     \end{align}
     where we used $\frac{\eta L_\eta}{2}<1$ and $\mu\eta\paren{1+\lambda_N^W-\eta\paren{L+\frac{2}{N\gamma}}}\in (0,1)$ for any
     \begin{equation*}
         0<\gamma <\frac{2N\eta}{\paren{1+\lambda_N^W-\eta L}-\frac{1}{\mu\eta}}.
     \end{equation*}
     Therefore, from a recursion and simplification, we obtain
     \begin{equation*}
         \E\normsq{x^{(k)}-x^*_{\eta,\gamma}}\le \paren{1-\mu\eta\paren{1+\lambda_N^W-\eta L_\gamma}}^k\normsq{x^{(0)}-x^*_{\eta,\gamma}}+\frac{\eta \sigma^2 N+2dN}{\mu\paren{1+\lambda_N^W-\eta L_\gamma}}.
     \end{equation*}
     Now we can obtain a complete bound for equation~\eqref{eq:gradbound1} as follows:
     \begin{align*}
         \E\normsq{\nabla \ug\paren{x^{(k)}}}&\le 2 L_{\gamma}^2\E\normsq{x^{(k)}-x_{\eta,\gamma}^*}+4L_{\gamma}^2\normsq{x^*_{\eta,\gamma}-x^*_\gamma}+4\normsq{\nabla U^{\gamma}(x^*_\gamma)}\\
         &\le 2L_\gamma^2\paren{1-\mu\eta\paren{1+\lambda_N^W-\eta L_\gamma}}^k\normsq{x^{(0)}-x^*_{\eta,\gamma}}+\frac{2L_\gamma^2(\eta \sigma^2 N+2dN)}{\mu\paren{1+\lambda_N^W-\eta L_\gamma}}\\
         &\qquad+4L_{\gamma}^2\normsq{x^*_{\eta,\gamma}-x^*_\gamma}+4\normsq{\nabla U^{\gamma}(x^*_\gamma)}\\
         &\le 4L_\gamma^2\paren{1-\mu\eta\paren{1+\lambda_N^W-\eta L_\gamma}}^k\normsq{x^{(0)}-x^*_{\gamma}}\\
         &\qquad\qquad+4L_\gamma^2\paren{1-\mu\eta\paren{1+\lambda_N^W-\eta L_\gamma}}^k\normsq{x^*_{\eta,\gamma}-x^*_\gamma}\\
         &\qquad\qquad+4L_{\gamma}^2\normsq{x^*_{\eta,\gamma}-x^*_\gamma}+4\normsq{\nabla U^{\gamma}(x^*_\gamma)}+\frac{2L_\gamma^2(\eta \sigma^2 N+2dN)}{\mu\paren{1+\lambda_N^W-\eta L_\gamma}}\\
         &\le 4L_\gamma^2\paren{1-\mu\eta\paren{1+\lambda_N^W-\eta L_\gamma}}^k\normsq{x^{(0)}-x^*_{\gamma}} + 8L_\gamma^2\normsq{x^*_{\eta,\gamma}-x^*_\gamma}\\
         &\qquad\qquad+4\normsq{\nabla U^{\gamma}(x^*_\gamma)}+\frac{2L_\gamma^2(\eta \sigma^2 N+2dN)}{\mu\paren{1+\lambda_N^W-\eta L_\gamma}}.
     \end{align*}
     Finally, we use Lemma~\ref{lem:xetastar-xstar} to complete the proof:
     \begin{align*}
         \E\normsq{\nabla \ug\paren{x^{(k)}}}&\le4L_\gamma^2\paren{1-\mu\eta\paren{1+\lambda_N^W-\eta L_\gamma}}^k\normsq{x^{(0)}-x^*_{\gamma}} + 8L_\gamma^2\frac{C_1^2\eta^2 N}{(1-\rho)^2}\\
         &\qquad\qquad+4\normsq{\nabla U^{\gamma}(x^*_\gamma)}+\frac{2L_\gamma^2(\eta \sigma^2 N+2dN)}{\mu\paren{1+\lambda_N^W-\eta L_\gamma}}.
     \end{align*}
     The proof is complete.
\end{proof}

\subsection{Proof of Lemma~\ref{lem:agent-avg}}
\label{lemproof:agent-avg}
\begin{proof}
    Recall the concatenated version of the algorithm from equation~\eqref{eq:concatenatedxk}:
    \begin{equation*}
        x^{(k+1)}=\Wc x^{(k)} -\eta \nabla U^{\gamma}\paren{x^{(k)}}-\eta \xi^{(k+1)}+\sqrt{2\eta}w^{(k+1)},
    \end{equation*}
    where $\Wc=W\otimes \mb{I}_d$. From a telescoping computation, we obtain that
    \begin{align}
        x^{(k)}&=\paren{W^k\otimes \mb{I}_d}x^{(0)} - \eta \sum_{r=0}^{k-1}\paren{W^{k-1-r}\otimes \mb{I}_d}\nabla\ug\paren{x^{(r)}}\nonumber\\
        &\qquad-\eta \sum_{r=0}^{k-1}\paren{W^{k-1-r}\otimes \mb{I}_d} \xi^{r+1} +\sqrt{2\eta}\sum_{r=0}^{k-1}\paren{W^{k-1-r}\otimes \mb{I}_d}w^{(s+1)}.
    \end{align}
    Next, define the concatenated average $\bar{\mb{x}}\in \R^{Nd}$ as
    \begin{equation*}       \bar{\mb{x}}:=\left[\left(\bar{x}^{(k)}\right)^{\top},\cdots, \left(\bar{x}^{(k)}\right)^{\top}\right]^{\top}=\frac1N\sqb{\paren{\mb{1}_N\mb{1}_N^{\top}}\otimes \mb{I}_d}x^{(k)}.  
    \end{equation*}
    Therefore, 
    \begin{align*}
        &x^{(k)}-\frac1N\sqb{\paren{\mb{1}_N\mb{1}_N^{\top}}\otimes \mb{I}_d}x^{(k)}\\
        &=\paren{W^k\otimes \mb{I}_d}x^{(0)}-\frac1N\sqb{\paren{\mb{1}_N\mb{1}_N^{\top}W^k}\otimes \mb{I}_d}x^{(0)}\\
        &-\eta\sum_{r=0}^{k-1}\paren{W^{k-1-r}\otimes \mb{I}_d}\nabla \ug\paren{x^{(r)}} +\eta \sum_{r=0}^{k-1}\frac1N\sqb{\paren{\mb{1}_N\mb{1}_N^{\top}W^{k-1-r}}\otimes \mb{I}_d}\ug\paren{x^{(r)}}\\
        &-\eta\sum_{r=0}^{k-1}\paren{W^{k-1-r}\otimes \mb{I}_d}\nabla \xi^{(r+1)} +\eta \sum_{r=0}^{k-1}\frac1N\sqb{\paren{\mb{1}_N\mb{1}_N^{\top}W^{k-1-r}}\otimes \mb{I}_d}\xi^{(r+1)}\\
        &+\sqrt{2\eta}\sum_{r=0}^{k-1}\paren{W^{k-1-r}\otimes \mb{I}_d}\nabla w^{(r+1)} -\sqrt{2\eta} \sum_{r=0}^{k-1}\frac1N\sqb{\paren{\mb{1}_N\mb{1}_N^{\top}W^{k-1-r}}\otimes \mb{I}_d}w^{(r+1)}.
    \end{align*}
    Since $W$ is doubly stochastic, 
    \[
    \frac1N\sqb{\paren{\mb{1}_N\mb{1}_N^{\top}W^k}\otimes \mb{I}_d}=\frac1N\sqb{\paren{\mb{1}_N\mb{1}_N^{\top}}\otimes \mb{I}_d},\qquad\text{for all}\quad k\ge 1.
    \]
    Now applying the Cauchy-Schwarz inequality, we have
    \begin{align*}
        &\normsq{x^{(k)}-\frac1N\sqb{\paren{\mb{1}_N\mb{1}_N^{\top}}\otimes \mb{I}_d}x^{(k)}}\\
        &\le 4\normsq{\paren{W^k\otimes \mb{I}_d}x^{(0)}-\frac1N\sqb{\paren{\mb{1}_N\mb{1}_N^{\top}W^k}\otimes \mb{I}_d}x^{(0)}}\\
        &+4\normsq{-\eta\sum_{r=0}^{k-1}\paren{W^{k-1-r}\otimes \mb{I}_d}\nabla \ug\paren{x^{(r)}} +\eta \sum_{r=0}^{k-1}\frac1N\sqb{\paren{\mb{1}_N\mb{1}_N^{\top}W^{k-1-r}}\otimes \mb{I}_d}\ug\paren{x^{(r)}}}\\
        &+4\normsq{-\eta\sum_{r=0}^{k-1}\paren{W^{k-1-r}\otimes \mb{I}_d}\nabla \xi^{(r+1)} +\eta \sum_{r=0}^{k-1}\frac1N\sqb{\paren{\mb{1}_N\mb{1}_N^{\top}W^{k-1-r}}\otimes \mb{I}_d}\xi^{(r+1)}}\\
        &+4\normsq{\sqrt{2\eta}\sum_{r=0}^{k-1}\paren{W^{k-1-r}\otimes \mb{I}_d}\nabla w^{(r+1)} -\sqrt{2\eta} \sum_{r=0}^{k-1}\frac1N\sqb{\paren{\mb{1}_N\mb{1}_N^{\top}W^{k-1-r}}\otimes \mb{I}_d}w^{(r+1)}}\\
        &=4\normsq{\paren{W^k\otimes \mb{I}_d}x^{(0)}-\frac1N\sqb{\paren{\mb{1}_N\mb{1}_N^{\top}}\otimes \mb{I}_d}x^{(0)}}\\
        &+4\normsq{-\eta\sum_{r=0}^{k-1}\paren{W^{k-1-r}\otimes \mb{I}_d}\nabla \ug\paren{x^{(r)}} +\eta \sum_{r=0}^{k-1}\frac1N\sqb{\paren{\mb{1}_N\mb{1}_N^{\top}}\otimes \mb{I}_d}\ug\paren{x^{(r)}}}\\
        &+4\normsq{-\eta\sum_{r=0}^{k-1}\paren{W^{k-1-r}\otimes \mb{I}_d}\nabla \xi^{(r+1)} +\eta \sum_{r=0}^{k-1}\frac1N\sqb{\paren{\mb{1}_N\mb{1}_N^{\top}}\otimes \mb{I}_d}\xi^{(r+1)}}\\
        &+4\normsq{\sqrt{2\eta}\sum_{r=0}^{k-1}\paren{W^{k-1-r}\otimes \mb{I}_d}\nabla w^{(r+1)} -\sqrt{2\eta} \sum_{r=0}^{k-1}\frac1N\sqb{\paren{\mb{1}_N\mb{1}_N^{\top}}\otimes \mb{I}_d}w^{(r+1)}}.
    \end{align*}
    Factoring the common terms,
    \begin{align}
        &\normsq{x^{(k)}-\frac1N\sqb{\paren{\mb{1}_N\mb{1}_N^{\top}}\otimes \mb{I}_d}x^{(k)}}\nonumber\\
        &\le 4\normsq{\bsqb{\paren{W^k-\frac1N\mb{1}_N\mb{1}_N^{\top}}\otimes \mb{I}_d}x^{(0)}}
        \nonumber\\
        &\qquad+4\eta^2 \normsq{\sum_{r=0}^{k-1}\bsqb{\paren{W^{k-1-r}-\frac1N\mb{1}_N\mb{1}_N^{\top}}\otimes \mb{I}_d}\nabla\ug\paren{x^{(r)}}}\nonumber\\
        &\qquad\qquad+4\eta^2\normsq{\sum_{r=0}^{k-1}\bsqb{\paren{W^{k-1-r}-\frac1N\mb{1}_N\mb{1}_N^{\top}}\otimes \mb{I}_d}\xi^{(r+1)}}\nonumber\\
        &\qquad\qquad\qquad+8\eta\normsq{\sum_{r=0}^{k-1}\bsqb{\paren{W^{k-1-r}-\frac1N\mb{1}_N\mb{1}_N^{\top}}\otimes \mb{I}_d}w^{(r+1)}}.\label{eq:jj}
    \end{align}
    Recall that $W$ has the spectral gap $1-\rho$ with $\rho:=\max\{|\lambda_2^W|,|\lambda_N^W|\}$; therefore, $W^k$ has the eigenvalues $(\lambda_i^W)^k$ with $1=\lambda_1^W>\lambda_2^W>\cdots>\lambda_N^W>-1$ for any $k\ge 1$. This results 
    \begin{align*}
        \norm{W^{k-1-r}-\frac1N\mb{1}_N\mb{1}_N^{\top}}&=\max\left\{|\lambda_2^W|^{k-1-r},|\lambda_N^W|^{k-1-r}\right\}=\rho^{k-1-r}\\
        \norm{W^k-\frac1N\mb{1}_N\mb{1}_N^{\top}}&=\max\left\{|\lambda_2^W|^{k},|\lambda_N^W|^{k}\right\}=\rho^k.
    \end{align*}
    Therefore,
    \begin{align}
        &4\eta^2 \normsq{\sum_{r=0}^{k-1}\bsqb{\paren{W^{k-1-r}-\frac1N\mb{1}_N\mb{1}_N^{\top}}\otimes \mb{I}_d}\nabla\ug\paren{x^{(r)}}}\nonumber\\
        &\le 4\eta^2\bsqb{\sum_{r=0}^{k-1}\norm{\paren{W^{k-1-r}-\frac1N\mb{1}_N\mb{1}_N^{\top}}\otimes \mb{I}_d}\cdot\norm{\nabla \ug\paren{x^{(r)}}}}^2\nonumber\\
        &\le 4\eta^2\bsqb{\sum_{r=0}^{k-1}\norm{\paren{W^{k-1-r}-\frac1N\mb{1}_N\mb{1}_N^{\top}}}\cdot\norm{\nabla \ug\paren{x^{(r)}}}}^2\nonumber\\
        &=4\eta^2\bsqb{\sum_{r=0}^{k-1}\rho^{k-1-r}\cdot\norm{\nabla\ug\paren{x^{(r)}}}}^2\nonumber\\
        &=4\eta^2\paren{\sum_{r=0}^{k-1}\rho^{k-1-r}}^2\paren{\frac{\sum_{r=0}^{k-1}\rho^{k-1-r}\nabla\ug\paren{x^{(r)}}}{\sum_{r=0}^{k-1}\rho^{k-1-r}}}^2.\label{eq:jj1}
    \end{align}
    For a real convex function $\Psi:\R\to \R$, for any points $(x_i)_{i=1}^n$ in the domain of $\Psi$, and for any $a_i>0$, Jensen's inequality implies that
    \[
    \Psi\paren{\frac{\sum_{i=1}^{n} a_ix_i}{\sum_{i=1}^{n} a_i}}\le \frac{\sum_{i=1}^{n} a_i\Psi(x_i)}{\sum_{i=1}^{n} a_i}.
    \]
    Using Jensen's inequality in~\eqref{eq:jj1}, we obtain
    \begin{align*}
        &4\eta^2 \normsq{\sum_{r=0}^{k-1}\bsqb{\paren{W^{k-1-r}-\frac1N\mb{1}_N\mb{1}_N^{\top}}\otimes \mb{I}_d}\nabla\ug\paren{x^{(r)}}}\\
        &\le4\eta^2\paren{\sum_{r=0}^{k-1}\rho^{k-1-r}}^2\sum_{r=0}^{k-1}\frac{\rho^{k-1-r}}{\sum_{r=0}^{k-1}\rho^{k-1-r}}\normsq{\nabla\ug\paren{x^{(r)}}}.
    \end{align*}
    Similarly, we compute that
    \begin{equation*}
        4\normsq{\bsqb{\paren{W^k-\frac1N\mb{1}_N\mb{1}_N^{\top}}\otimes \mb{I}_d}x^{(0)}}\le 4\normsq{\paren{W^k-\frac1N\mb{1}_N\mb{1}_N^{\top}}\otimes \mb{I}_d}\normsq{x^{(0)}}\le 4\rho^{2k}\normsq{x^{(0)}}.
    \end{equation*}
    Now we are ready to compute a bound between the individual iterates and their mean
    \begin{align}
        &\sum_{i=1}^N\E\normsq{x_i^{(k)}-\bar{x}^{(k)}}\nonumber\\
        &=\E\normsq{x^{(k)}-\frac1N\sqb{\paren{\mb{1}_N\mb{1}_N^{\top}}\otimes \mb{I}_d}x^{(k)}}\nonumber\\
        &\le 4\rho^{2k}\E\normsq{x^{(0)}}+4\eta^2\E\normsq{\paren{\sum_{r=0}^{k-1}\rho^{k-1-r}}^2\sum_{r=0}^{k-1}\frac{\rho^{k-1-r}}{\sum_{r=0}^{k-1}\rho^{k-1-r}}\normsq{\nabla\ug\paren{x^{(r)}}}}\nonumber\\
        &\qquad+4\eta^2\E\normsq{\sum_{r=0}^{k-1}\bsqb{\paren{W^{k-1-r}-\frac1N\mb{1}_N\mb{1}_N^{\top}}\otimes \mb{I}_d}\xi^{(r+1)}}\nonumber\\
        &\qquad\qquad+8\eta\E\normsq{\sum_{r=0}^{k-1}\bsqb{\paren{W^{k-1-r}-\frac1N\mb{1}_N\mb{1}_N^{\top}}\otimes \mb{I}_d}w^{(r+1)}}.\label{eq:jj2}
    \end{align}
    Recall from Lemma~\ref{lem:gradbound} that for every $k=1,2,3\cdots$,
    \[
    \E\normsq{\nabla\ug\paren{x^{(k)}}}\le D_\gamma^2,
    \]
    where $D_\gamma^2$ is defined in equation~\eqref{eq:D}:
    \begin{equation*}
        \begin{split}
            D_\gamma^2&:=4L_\gamma^2\paren{1-\mu\eta\paren{1+\lambda_N^W-\eta L_\gamma}}^k\normsq{x^{(0)}-x^*_{\gamma}} + 8L_\gamma^2\frac{C_1^2\eta^2 N}{(1-\rho)^2}\\
         &\qquad\qquad+4\normsq{\nabla U^{\gamma}(x^*_\gamma)}+\frac{2L_\gamma^2(\eta \sigma^2 N+2dN)}{\mu\paren{1+\lambda_N^W-\eta L_\gamma}}.
        \end{split}
    \end{equation*}
    Therefore, the second term on the right hand side of~\eqref{eq:jj2}, 
    \begin{align}
        &4\eta^2\E\normsq{\paren{\sum_{r=0}^{k-1}\rho^{k-1-r}}^2\sum_{r=0}^{k-1}\frac{\rho^{k-1-r}}{\sum_{r=0}^{k-1}\rho^{k-1-r}}\normsq{\nabla\ug\paren{x^{(r)}}}}\nonumber\\
        &\qquad\qquad\le 4\eta^2 D_\gamma^2\paren{\sum_{r=0}^{k-1}\rho^{k-1-r}}^2\sum_{r=0}^{k-1}\frac{\rho^{k-1-r}}{\sum_{r=0}^{k-1}\rho^{k-1-r}}\nonumber\\
        &\qquad\qquad=4\eta^2D_\gamma^2\paren{\frac{1-\rho^{k-1}}{1-\rho}}^2\sum_{r=0}^{k-1}\frac{\rho^{k-1-r}}{\frac{1-\rho^{k-1}}{1-\rho}}=4\eta^2D_\gamma^2\paren{\frac{1-\rho^{k-1}}{1-\rho}}\sum_{r=0}^{k-1}\rho^{k-1-r}\nonumber\\
        &\qquad\qquad=4\eta^2D_\gamma^2\frac{\paren{1-\rho^{k-1}}^2}{(1-\rho)^2}\le 4\eta^2D_\gamma^2\frac{1}{(1-\rho)^2}.\label{eq:jj3}
    \end{align}
    Using the result from~\eqref{eq:jj3} and Assumption~\ref{assump:gradnoise} in~\eqref{eq:jj2}, we obtain:
    \begin{align}
        &\sum_{i=1}^N\E\normsq{x_i^{(k)}-\bar{x}^{(k)}}\nonumber\\
        &\le 4\rho^{2k}\E\normsq{x^{(0)}}+ 4\eta^2D_\gamma^2\frac{1}{(1-\rho)^2}+4\eta^2\sum_{r=0}^{k-1}\normsq{W^{k-1-r}-\frac1N\mb{1}_N\mb{1}_N^{\top}}\E\normsq{\xi^{(r+1)}}\nonumber\\
        &\qquad+8\eta\sum_{r=0}^{k-1}\normsq{W^{k-1-r}-\frac1N\mb{1}_N\mb{1}_N^{\top}}\E\normsq{w^{(r+1)}}\nonumber\\
        &\le 4\rho^{2k}\E\normsq{x^{(0)}}+ 4\eta^2D_\gamma^2\frac{1}{(1-\rho)^2}+4\eta^2 N\sigma^2\sum_{r=0}^{k-1}\rho^{2(k-1-r)}+8\eta d N\sum_{r=0}^{k-1}\rho^{2(k-1-r)}\nonumber\\
        &\le 4\rho^{2k}\E\normsq{x^{(0)}}+ \frac{4\eta^2D_\gamma^2}{(1-\rho)^2}+\frac{4\eta^2 N\sigma^2}{(1-\rho^2)}+\frac{8\eta d N}{(1-\rho^2)}.\nonumber
    \end{align}
    Therefore,
    \[
\frac1N\sum_{i=1}^N\E\normsq{x_i^{(k)}-\bar{x}^{(k)}}\le \frac{4\rho^{2k}}{N}\E\normsq{x^{(0)}}+ \frac{4\eta^2D_\gamma^2}{N(1-\rho)^2}+\frac{4\eta^2\sigma^2}{(1-\rho^2)}+\frac{8\eta d }{(1-\rho^2)}.
    \]
    This completes the proof.
\end{proof}

\subsection{Proof of Lemma~\ref{lem:errbound}}
\label{lemproof:errorbound}
\begin{proof}
    By taking the $L_2$ norm of the error term, we get
    \begin{align}
&\E\normsq{\mc{E}_{k+1}}\nonumber\\
        &=\E\normsq{\frac1N\sum_{i=1}^N\bsqb{\nabla f_i\paren{\bar{x}^{(k)}-\nabla f_i\paren{x_i^{(k)}}}} +\frac{1}{N^{2}\gamma}\sum_{i=1}^N\bsqb{\proj\paren{\bar{x}^{(k)}-\proj\paren{x_i^{(k)}}}} }\nonumber\\
        &\le 2\E \normsq{\frac1N\sum_{i=1}^N\bsqb{\nabla f_i\paren{\bar{x}^{(k)}-\nabla f_i\paren{x_i^{(k)}}}}}+2\E\normsq{\frac{1}{N^{2}\gamma}\sum_{i=1}^N\bsqb{\proj\paren{\bar{x}^{(k)}-\proj\paren{x_i^{(k)}}}}}\nonumber\\
        &\le \frac{2}{N^2}\sum_{i=1}^NN\E\normsq{\nabla f_i\paren{\bar{x}^{(k)}-\nabla f_i\paren{x_i^{(k)}}}}+\frac{2}{N^4\gamma^2}\sum_{i=1}^NN\E\normsq{\proj\paren{\bar{x}^{(k)}-\proj\paren{x_i^{(k)}}}}\nonumber\\
        &\le \frac{2L^2}{N}\sum_{i=1}^N\E\normsq{x_i^{(k)}-\bar{x}^{(k)}}+\frac{8}{N^{3}\gamma^2}\sum_{i=1}^N\E\normsq{x_i^{(k)}-\bar{x}^{(k)}}\nonumber\\
        &=2\paren{L^2+\frac{4}{N^{2}\gamma^2}}\frac1N\sum_{i=1}^N\E\normsq{x_i^{(k)}-\bar{x}^{(k)}}\nonumber\\
        &\le 2\paren{L+\frac{2}{N\gamma}}^2\frac1N\sum_{i=1}^N\E\normsq{x_i^{(k)}-\bar{x}^{(k)}}\nonumber\\
        &=2L_\gamma^2\frac1N\sum_{i=1}^N\E\normsq{x_i^{(k)}-\bar{x}^{(k)}}.
        \label{eq:errbound1}
    \end{align}
    Using the results from Lemma~\ref{lem:agent-avg}, we obtain
    \begin{align}
        \E\normsq{\mc{E}_{k+1}}&\le \frac{8L_\gamma^2\rho^{2k}}{N}\E\normsq{x^{(0)}}+\frac{8L_\gamma^2\eta^2D_\gamma^2}{N(1-\rho)^2}+\frac{8L_\gamma^2\eta^2\sigma^2}{(1-\rho^2)}+\frac{16L_\gamma^2\eta d }{(1-\rho^2)}.
    \end{align}
    This completes the proof.
\end{proof}

\subsection{Proof of Lemma~\ref{lem:myula-avg-bound}}
\label{lemproof:myula-avg-bound}

\begin{proof}
    From~\eqref{eq:meanite-again} and \eqref{eq:myulac} we obtain:
    \begin{equation*}
        \bar{x}^{(k+1)}-\tilde{x}^{(k+1)}=\bar{x}^{(k)}-\tilde{x}^{(k)}-\eta\sqb{\nabla \hat{u}^\gamma\paren{\bar{x}^{(k)}}-\nabla \hat{u}^\gamma\paren{\tilde{x}^{(k)}}}+\eta\mc{E}_{k+1}-\eta\bar{\xi}^{(k+1)}.
    \end{equation*}
    Therefore, 
    \begin{align}
        &\normsq{\bar{x}^{(k+1)}-\tilde{x}^{(k+1)}}\nonumber\\
        &=\normsq{\bar{x}^{(k)}-\tilde{x}^{(k)}-\eta\sqb{\nabla \hat{u}^\gamma\paren{\bar{x}^{(k)}}-\nabla \hat{u}^\gamma\paren{\tilde{x}^{(k)}}}+\eta\mc{E}_{k+1}-\eta\bar{\xi}^{(k+1)}}\nonumber\\
        &=\normsq{\bar{x}^{(k)}-\tilde{x}^{(k)}-\eta\sqb{\nabla \hat{u}^\gamma\paren{\bar{x}^{(k)}}-\nabla \hat{u}^\gamma\paren{\tilde{x}^{(k)}}}}+\eta^2\normsq{\mc{E}_{k+1}-\bar{\xi}^{(k+1)}}\nonumber\\
        &\qquad+2\lrangle{\bar{x}^{(k)}-\tilde{x}^{(k)}-\eta\sqb{\nabla \hat{u}^\gamma\paren{\bar{x}^{(k)}}-\nabla \hat{u}^\gamma\paren{\tilde{x}^{(k)}}},\eta\mc{E}_{k+1}-\eta\bar{\xi}^{(k+1)}}\nonumber\\
        &=\normsq{\bar{x}^{(k)}-\tilde{x}^{(k)}}+\eta^2\normsq{\nabla \hat{u}^\gamma\paren{\bar{x}^{(k)}}-\nabla \hat{u}^\gamma\paren{\tilde{x}^{(k)}}}+\eta^2\normsq{\mc{E}_{k+1}-\bar{\xi}^{(k+1)}}\nonumber\\
        &\qquad-2\lrangle{\bar{x}^{(k)}-\tilde{x}^{(k)},\eta\sqb{\nabla \hat{u}^\gamma\paren{\bar{x}^{(k)}}-\nabla \hat{u}^\gamma\paren{\tilde{x}^{(k)}}}}\nonumber\\
        &\qquad+2\lrangle{\bar{x}^{(k)}-\tilde{x}^{(k)}-\eta\sqb{\nabla \hat{u}^\gamma\paren{\bar{x}^{(k)}}-\nabla \hat{u}^\gamma\paren{\tilde{x}^{(k)}}},\eta\mc{E}_{k+1}-\eta\bar{\xi}^{(k+1)}}\nonumber\\
        &\le \normsq{\bar{x}^{(k)}-\tilde{x}^{(k)}}+\eta^2\frac{L_\gamma}{N}\lrangle{\bar{x}^{(k)}-\tilde{x}^{(k)},\nabla \hat{u}^\gamma\paren{\bar{x}^{(k)}}-\nabla \hat{u}^\gamma\paren{\tilde{x}^{(k)}}}\nonumber\\
        &\qquad-2\lrangle{\bar{x}^{(k)}-\tilde{x}^{(k)},\eta\sqb{\nabla \hat{u}^\gamma\paren{\bar{x}^{(k)}}-\nabla \hat{u}^\gamma\paren{\tilde{x}^{(k)}}}}+\eta^2\normsq{\mc{E}_{k+1}-\bar{\xi}^{(k+1)}}\nonumber\\
        &\qquad+2\lrangle{\bar{x}^{(k)}-\tilde{x}^{(k)}-\eta\sqb{\nabla \hat{u}^\gamma\paren{\bar{x}^{(k)}}-\nabla \hat{u}^\gamma\paren{\tilde{x}^{(k)}}},\eta\mc{E}_{k+1}-\eta\bar{\xi}^{(k+1)}}\nonumber\\
        &\le \paren{1-\frac{2\mu\eta}{N}\paren{1-\frac{\eta L_\gamma}{2N}}}\normsq{\bar{x}^{(k)}-\tilde{x}^{(k)}}+\eta^2\normsq{\mc{E}_{k+1}-\bar{\xi}^{(k+1)}}\nonumber\\
        &\qquad+2\lrangle{\bar{x}^{(k)}-\tilde{x}^{(k)}-\eta\sqb{\nabla \hat{u}^\gamma\paren{\bar{x}^{(k)}}-\nabla \hat{u}^\gamma\paren{\tilde{x}^{(k)}}},\eta\mc{E}_{k+1}-\eta\bar{\xi}^{(k+1)}},\label{eq:myula2}
    \end{align}
    where we used $L_\gamma/N$-smoothness and $\mu/N$-strongly convexity of $\hat{u}^\gamma(x)$, and the assumption that $\eta<\frac{2N}{L_\gamma}=\frac{2N}{L+\frac{2}{N\gamma}}$ in obtaining~\eqref{eq:myula2}. Also, by Assumption~\ref{assump:gradnoise}, the term $\bar{\xi}^{(k+1)}$ has mean zero conditional on the natural filtration of the iterates till time $k$. Thus, by taking expectations, it yields that
    \begin{align}
        &\E\normsq{\bar{x}^{(k+1)}-\tilde{x}^{(k+1)}}\nonumber\\
        &\le \paren{1-\frac{2\mu\eta}{N}\paren{1-\frac{\eta L_\gamma}{2N}}}\E\normsq{\bar{x}^{(k)}-\tilde{x}^{(k)}}+\eta^2\E\normsq{\mc{E}_{k+1}-\bar{\xi}^{(k+1)}}\nonumber\\
        &\qquad+2\E\lrangle{\bar{x}^{(k)}-\tilde{x}^{(k)}-\eta\sqb{\nabla \hat{u}^\gamma\paren{\bar{x}^{(k)}}-\nabla \hat{u}^\gamma\paren{\tilde{x}^{(k)}}},\eta\mc{E}_{k+1}-\eta\bar{\xi}^{(k+1)}}\nonumber\\
        &=\paren{1-\frac{2\mu\eta}{N}\paren{1-\frac{\eta L_\gamma}{2N}}}\E\normsq{\bar{x}^{(k)}-\tilde{x}^{(k)}}+\eta^2\E\normsq{\mc{E}_{k+1}}+\eta^2\E\normsq{\bar{\xi}^{(k+1)}}\nonumber\\
        &\qquad+2\E\lrangle{\bar{x}^{(k)}-\tilde{x}^{(k)}-\eta\sqb{\nabla \hat{u}^\gamma\paren{\bar{x}^{(k)}}-\nabla \hat{u}^\gamma\paren{\tilde{x}^{(k)}}},\eta\mc{E}_{k+1}-\eta\bar{\xi}^{(k+1)}}\nonumber\\
        &\le \paren{1-\frac{2\mu\eta}{N}\paren{1-\frac{\eta L_\gamma}{2N}}}\E\normsq{\bar{x}^{(k)}-\tilde{x}^{(k)}}+\eta^2\E\normsq{\mc{E}_{k+1}}+\eta^2\E\normsq{\bar{\xi}^{(k+1)}}\nonumber\\
        &\qquad+2\paren{1+\frac{\eta L_\gamma}{N}}\eta\E\bsqb{\norm{\bar{x}^{(k)}-\tilde{x}^{(k)}}\cdot\norm{\mc{E}_{k+1}}}.\label{eq:myula3}
    \end{align}
    For any $a,b\ge 0$ and $\varepsilon>0$, we have the Young's inequality $\displaystyle 2ab\le \varepsilon a^2+\frac{b^2}{\varepsilon}$. Choose
    \[
    a=\norm{\bar{x}^{(k)}-\tilde{x}^{(k)}},
    \qquad
    b=\norm{\mc{E}_{k+1}},
    \qquad
    \varepsilon = \frac{\frac{\mu}{N}\paren{1-\frac{\eta L_\gamma}{2}}}{1+\frac{\eta L_\gamma}{N}}.
    \]
    Then, from~\eqref{eq:myula3}, we obtain
    \begin{align}
        &\E\normsq{\bar{x}^{(k+1)}-\tilde{x}^{(k+1)}}\nonumber\\
        &\le \paren{1-\frac{2\mu\eta}{N}\paren{1-\frac{\eta L_\gamma}{2N}}}\E\normsq{\bar{x}^{(k)}-\tilde{x}^{(k)}}+\eta^2\E\normsq{\mc{E}_{k+1}}+\eta^2\frac{\sigma^2}{N}\nonumber\\
        &\qquad+\left(1+\frac{\eta L_\gamma}{N}\right)\eta \paren{\frac{\mu\paren{1-\frac{\eta L_\gamma}{2N}}}{N+\eta L_\gamma}\E\normsq{\bar{x}^{(k)}-\tilde{x}^{(k)}}+\frac{N+\eta L_\gamma}{\mu\paren{1-\frac{\eta L_\gamma}{2N}}}\E\normsq{\mc{E}_{k+1}}}\nonumber\\
        &=\paren{1-\frac{\mu\eta}{N}\paren{1-\frac{\eta L_\gamma}{2N}}}\E\normsq{\bar{x}^{(k)}-\tilde{x}^{(k)}}+\eta\paren{\eta+\frac{(1+ \frac{\eta L_\gamma}{N})^2}{\frac{\mu}{N}\paren{1-\frac{\eta L_\gamma}{2N}}}}\E\normsq{\mc{E}_{k+1}}+\eta^2\frac{\sigma^2}{N}.
        \label{eq:myula4}
    \end{align}
    Here, we assume that the leading term $1-\frac{\mu\eta}{N}\paren{1-\frac{\eta L_\gamma}{2N}}\in [0,1)$. Now using the results from Lemma~\ref{lem:errbound}, we obtain
    \begin{align}
        &\E\normsq{\bar{x}^{(k+1)}-\tilde{x}^{(k+1)}}\nonumber\\
        &\le \paren{1-\frac{\mu\eta}{N}\paren{1-\frac{\eta L_\gamma}{2N}}}\E\normsq{\bar{x}^{(k)}-\tilde{x}^{(k)}}\nonumber\\
        &\quad+\eta\paren{\eta+\frac{(1+\frac{\eta L_\gamma}{N})^2}{\frac{\mu}{N}\paren{1-\frac{\eta L_\gamma}{2N}}}}\paren{\frac{8L_\gamma^2\rho^{2k}}{N}\E\normsq{x^{(0)}}+\frac{8L_\gamma^2\eta^2D_\gamma^2}{N(1-\rho)^2}+\frac{8L_\gamma^2\eta^2\sigma^2}{(1-\rho^2)}+\frac{16L_\gamma^2\eta d }{(1-\rho^2)}}\nonumber
        \\
        &\qquad+\eta^2\frac{\sigma^2}{N}.\nonumber
    \end{align}
    Since $\E\left\Vert\bar{x}^{(0)}-\tilde{x}^{(0)}\right\Vert=0$, by iterations,
    \begin{align*}
    &\E\normsq{\bar{x}^{(k)}-\tilde{x}^{(k)}}\nonumber\\
        &\le \sum_{r=0}^{k-1}\paren{1-\frac{\mu\eta}{N}\paren{1-\frac{\eta L_\gamma}{2N}}}^r\\
        &\qquad\cdot\bsqb{\eta\paren{\eta+\frac{(1+\frac{\eta L_\gamma}{N})^2}{\frac{\mu}{N}\paren{1-\frac{\eta L_\gamma}{2N}}}}\paren{\frac{8L_\gamma^2\eta^2D_\gamma^2}{N(1-\rho)^2}+\frac{8L_\gamma^2\eta^2\sigma^2}{(1-\rho^2)}+\frac{16L_\gamma^2\eta d }{(1-\rho^2)}}+\eta^2\frac{\sigma^2}{N}}\\
        &\qquad+\sum_{r=0}^{k-1}\paren{1-\frac{\mu\eta}{N}\paren{1-\frac{\eta L_\gamma}{2N}}}^r\eta\paren{\eta+\frac{(1+\eta L_\gamma)^2}{\mu\paren{1-\frac{\eta L_\gamma}{2}}}}\frac{8L_\gamma^2\rho^{2(k-r)}}{N}\E\normsq{x^{(0)}}\\
        &=\frac{1-\paren{1-\frac{\mu\eta}{N}\paren{1-\frac{\eta L_\gamma}{2N}}}^k}{1-\paren{1-\frac{\mu\eta}{N}\paren{1-\frac{\eta L_\gamma}{2N}}}}\\
        &\qquad\cdot\bsqb{\eta\paren{\eta+\frac{(1+\frac{\eta L_\gamma}{N})^2}{\frac{\mu}{N}\paren{1-\frac{\eta L_\gamma}{2N}}}}\paren{\frac{8L_\gamma^2\eta^2D_\gamma^2}{N(1-\rho)^2}+\frac{8L_\gamma^2\eta^2\sigma^2}{(1-\rho^2)}+\frac{16L_\gamma^2\eta d }{(1-\rho^2)}}+\eta^2\frac{\sigma^2}{N}}\\
        &\qquad+\frac{\rho^{2k}-\paren{1-\frac{\mu\eta}{N}\paren{1-\frac{\eta L_\gamma}{2N}}}^k}{1-\paren{1-\frac{\mu\eta}{N}\paren{1-\frac{\eta L_\gamma}{2N}}}\frac{1}{\rho^2}}\frac{8L_\gamma^2}{N}\E\normsq{x^{(0)}}.
    \end{align*}
    Finally, for every $k\in\mathbb{N}$,
    \begin{align*}
        &\E\normsq{\bar{x}^{(k)}-\tilde{x}^{(k)}}\nonumber\\
        &\qquad\le \frac{\eta\paren{\eta+\frac{(1+\frac{\eta L_\gamma}{N})^2}{\frac{\mu}{N}\paren{1-\frac{\eta L_\gamma}{2N}}}}\paren{\frac{8L_\gamma^2\eta^2D_\gamma^2}{N(1-\rho)^2}+\frac{8L_\gamma^2\eta^2\sigma^2}{(1-\rho^2)}+\frac{16L_\gamma^2\eta d }{(1-\rho^2)}}+\eta^2\frac{\sigma^2}{N}}{\frac{\mu\eta}{N}\paren{1-\frac{\eta L_\gamma}{2N}}}\\
        &\qquad\qquad+\frac{\rho^{2k}-\paren{1-\frac{\mu\eta}{N}\paren{1-\frac{\eta L_\gamma}{2N}}}^k}{\rho^2-\paren{1-\frac{\mu\eta}{N}\paren{1-\frac{\eta L_\gamma}{2N}}}}\cdot\frac{8L_\gamma^2\rho^2}{N}\E\normsq{x^{(0)}}\\
        &\qquad= \frac{\eta\paren{\eta+\frac{(1+\frac{\eta L_\gamma}{N})^2}{\frac{\mu}{N}\paren{1-\frac{\eta L_\gamma}{2N}}}}\paren{\frac{8L_\gamma^2\eta D_\gamma^2}{N(1-\rho)^2}+\frac{8L_\gamma^2\eta\sigma^2}{(1-\rho^2)}+\frac{16L_\gamma^2  d }{(1-\rho^2)}}+\frac{\eta\sigma^2}{N}}{\frac{\mu}{N}\paren{1-\frac{\eta L_\gamma}{2N}}}\\
        &\qquad\qquad+\frac{\rho^{2k}-\paren{1-\frac{\mu\eta}{N}\paren{1-\frac{\eta L_\gamma}{2N}}}^k}{\rho^2-\paren{1-\frac{\mu\eta}{N}\paren{1-\frac{\eta L_\gamma}{2N}}}}\cdot\frac{8L_\gamma^2\rho^2}{N}\E\normsq{x^{(0)}}\\
        &\qquad=\eta \paren{\frac{\eta}{\frac{\mu}{N}\paren{1-\frac{\eta L_\gamma}{2N}}}+\frac{(1+\frac{\eta L_\gamma}{N})^2}{\frac{\mu^2}{N^{2}}\paren{1-\frac{\eta L_\gamma}{2N}}^2}}\paren{\frac{8L_\gamma^2\eta D_\gamma^2}{N(1-\rho)^2}+\frac{8L_\gamma^2\eta\sigma^2}{(1-\rho^2)}+\frac{16L_\gamma^2  d }{(1-\rho^2)}}\\
        &\qquad\qquad+\frac{\eta\sigma^2}{\mu\paren{1-\frac{\eta L_\gamma}{2N}}}+\frac{\rho^{2k}-\paren{1-\frac{\mu\eta}{N}\paren{1-\frac{\eta L_\gamma}{2N}}}^k}{\rho^2-\paren{1-\frac{\mu\eta}{N}\paren{1-\frac{\eta L_\gamma}{2N}}}}\cdot\frac{8L_\gamma^2\rho^2}{N}\E\normsq{x^{(0)}}.
    \end{align*}
    This completes the proof.
\end{proof}

\subsection{Proof of Lemma~\ref{lem:minimizer}}
\label{lemproof:minimizer}

\begin{proof}
    Define the process $Y_t:=X_t-\hat{x}_\gamma^*$, where $X_{t}$ follows the overdamped Langevin diffusion: 
    \begin{equation}\label{eqn:overdamped}
        dX_t:=-\nabla \hat{u}^\gamma(X_t)dt+\sqrt{2N^{-1}}dW_t,
    \end{equation}
    where $W_t$ is a standard $d$-dimensional Brownian motion, such that $X_{\infty}$ follows the unique stationary distribution $\pi^{\gamma}$ of \eqref{eqn:overdamped}. 
    Then, we have 
    $$
    dY_t=-\nabla \hat{u}^\gamma(X_t)dt+\sqrt{2N^{-1}}dW_t.
    $$
    Now by It\^{o}'s formula,
    \begin{align*}
        d\|Y_t\|^2&= -2\left\langle Y_t,\nabla \hat u^\gamma(X_t) - \nabla \hat u^\gamma(\hat{x}_\gamma^*)\right\rangle\,dt+2\sqrt{2N^{-1}}\,\langle Y_t,dW_t\rangle +(2N^{-1}d)\,dt.
    \end{align*}
    Multiplying both sides by the integrating factor $e^{\frac{\mu}{N} t}$ and using the product rule yields
    \[
    \begin{aligned}
        e^{\frac{\mu}{N} t}\|Y_t\|^2&=\|Y_0\|^2+2\sqrt{2N^{-1}}\int_0^t e^{\frac{\mu}{N} s}\langle Y_s,dW_s\rangle \\
        &\qquad-2\int_0^t e^{\frac{\mu}{N} s}\left\langle Y_s,\nabla \hat u^\gamma(X_s) - \nabla \hat u^\gamma\left(\hat x_\gamma^*\right)\right\rangle ds
        \\
        &\qquad\qquad+2N^{-1}d\int_0^t e^{\frac{\mu}{N} s}\,ds +\frac{\mu}{N} \int_0^t e^{\frac{\mu}{N} s}\|Y_s\|^2 ds.
    \end{aligned}
    \]
    Since $\hat u^\gamma$ is $\mu/N$-strongly convex, we have
    \[
        \left\langle x-y, \nabla \hat u^\gamma(x)-\nabla \hat u^\gamma(y)\right\rangle
        \ge\frac{\mu}{N} \|x-y\|^2.
    \]
    Applying this inequality with $x=X_s$ and $y=\hat x_\gamma^*$ implies
    \[
        \left\langle Y_s,
        \nabla \hat u^\gamma(X_s) - \nabla \hat u^\gamma(\hat x_\gamma^*)
        \right\rangle
        \ge
        \frac{\mu}{N} \|Y_s\|^2.
    \]
    Therefore,
    \[
        -2\int_0^t e^{\frac{\mu}{N} s}\left\langle Y_s,
        \nabla \hat u^\gamma(X_s) - \nabla \hat u^\gamma\left(\hat x_\gamma^*\right)
        \right\rangle ds+\frac{\mu}{N} \int_0^t e^{\frac{\mu}{N} s}\|Y_s\|^2 ds
        \le 0.
    \]
    Dropping this non-positive term yields
    \[
        e^{\frac{\mu}{N} t}\|Y_t\|^2
        \le
        \|Y_0\|^2
        +2\sqrt{2N^{-1}}\int_0^t e^{\frac{\mu}{N} s}\langle Y_s,dW_s\rangle
        +2N^{-1}d\int_0^t e^{\frac{\mu}{N} s}\,ds.
    \]
    Taking expectations and using the fact that the stochastic integral has mean zero gives
    \[
        \mathbb{E}\left[e^{\frac{\mu}{N} t}\|Y_t\|^2\right]
        \le
        \mathbb{E}\|Y_0\|^2
        +2N^{-1}d\int_0^t e^{\frac{\mu}{N} s}\,ds
        =
        \mathbb{E}\|Y_0\|^2
        +\frac{2N^{-1}d}{\frac{\mu}{N}}\left(e^{\frac{\mu}{N} t}-1\right).
    \]
    Dividing by $e^{\frac{\mu}{N} t}$ yields
    \[
        \mathbb{E}\|Y_t\|^2
        \le
        e^{-\frac{\mu}{N} t}\mathbb{E}\|Y_0\|^2
        +
        \frac{2N^{-1}d}{\frac{\mu}{N}}\left(1-e^{-\frac{\mu}{N} t}\right)
        \le
        e^{-\frac{\mu}{N} t}\mathbb{E}\|Y_0\|^2
        +
        \frac{2N^{-1}d}{\frac{\mu}{N}}.
    \]
    Letting $t\to\infty$ and using convergence to the invariant distribution $\pi^\gamma$ gives
    \[
        \mathbb{E}\|X_\infty - \hat x_\gamma^*\|^2 \le \frac{2N^{-1}d}{\frac{\mu}{N}}=\frac{2d}{\mu}.
    \]
    This completes the proof.
\end{proof}

\end{document}